\documentclass[final,12pt,authoryear]{elsarticle}

\makeatletter
\def\ps@pprintTitle{%
 \let\@oddhead\@empty
 \let\@evenhead\@empty
\def\@oddfoot{\footnotesize
      \itshape\hfill\today}%
     \let\@evenfoot\@oddfoot}

\DeclareGraphicsExtensions{.pdf,.jpeg,.png}
\usepackage[caption = false,font = footnotesize]{subfig}
\usepackage{cite}
\usepackage{url}
\usepackage{rotating}

\usepackage[cmex10]{amsmath}
\usepackage{amsfonts}
\usepackage{amssymb}
\usepackage{amsthm}
\usepackage{bm}
\usepackage{empheq}
\usepackage{booktabs}

\usepackage{algorithmic}
\usepackage{setspace}
\usepackage{array}

\usepackage{url}
\usepackage[table]{xcolor}
\usepackage{multirow}
\usepackage{booktabs}
\usepackage{flushend}
\usepackage[colorlinks=true]{hyperref}
\usepackage{float}

\newcommand{\vect}[1]{\mathbf{#1}} 
\newcommand{\norm}[2][2]{\left\lVert #2 \right\rVert_{#1}} 
\DeclareMathSymbol{\R}{\mathalpha}{AMSb}{"52} 

\floatstyle{ruled}
\newfloat{algorithm}{tbp}{loa}
\providecommand{\algorithmname}{Algorithm}
\floatname{algorithm}{\protect\algorithmname}

\definecolor{LightGray}{gray}{0.85}

\newtheorem{theorem}{Theorem}
\newtheorem{proposition}[theorem]{Proposition}

\newcommand*\widefbox[1]{\fbox{\hspace{0.5em}#1\hspace{0.5em}}}

\journal{Neural Networks}

\begin{document}

\begin{frontmatter}

\title{A Framework for Parallel and Distributed Training of Neural Networks}

\author[sapienza]{Simone Scardapane\corref{cor1}}
\ead{simone.scardapane@uniroma1.it}
\cortext[cor1]{Corresponding author. Phone: +39 06 44585495, Fax: +39 06 4873300.}

\author[perugia]{Paolo Di Lorenzo\corref{cor2}}
\ead{paolo.dilorenzo@unipg.it}
\cortext[cor2]{The work of Paolo Di Lorenzo was founded by the ``Fondazione Cassa di Risparmio di Perugia''}

\address[sapienza]{Department of Information Engineering, Electronics and Telecommunications, ``Sapienza'' University of Rome, Via Eudossiana 18, 00184 Rome, Italy}

\address[perugia]{Department of Engineering, University of Perugia, Via G. Duranti 93, 06125, Perugia, Italy}

\begin{abstract}
The aim of this paper is to develop a general framework for training neural networks (NNs) in a distributed environment, where
training data is partitioned over a set of agents that communicate with each other through a sparse, possibly time-varying, connectivity pattern. In such distributed scenario, the training problem can be formulated as the (regularized) optimization of a non-convex social cost function, given by the sum of local (non-convex) costs, where each agent contributes with a single error term defined with respect to its local dataset. To devise a flexible and efficient solution, we customize a recently proposed framework for non-convex optimization over networks, which hinges on a (primal) convexification-decomposition technique to handle non-convexity, and a dynamic consensus procedure to diffuse information among the agents. Several typical choices for the training criterion (e.g., squared loss, cross entropy, etc.) and regularization (e.g., $\ell_2$ norm, sparsity inducing penalties, etc.) are included in the framework and explored along the paper. Convergence to a stationary solution of the social non-convex problem is guaranteed under mild assumptions. Additionally, we show a principled way allowing each agent to exploit a possible multi-core architecture (e.g., a local cloud) in order to parallelize its local optimization step, resulting in strategies that are both distributed (across the agents) and parallel (inside each agent) in nature. A comprehensive set of experimental results validate the proposed approach.
\end{abstract}


{\footnotesize
       Published on \textbf{Neural Networks}, doi 10.1016/j.neunet.2017.04.004.\hfill\today}

\end{frontmatter}

\section{Introduction}
\label{sec:introduction}

We consider the problem of training a Neural Network (NN) model, when training data is distributed over different agents that are connected by a sparse, possibly time-varying, communication network. To grasp the main motivation, let us consider a `smart' environment, wherein thousands of low-power sensors (e.g., cameras, wearables, etc.) are embedded to provide context-aware assistance, security provisioning, and so forth \citep{pottie2000wireless,boric2002wireless}. If the amount of produced data is small and we can count on a very reliable communication network, we may think of a centralized approach where all the data are transmitted to one (or more) fusion center that performs the learning task. However, in big data applications, sharing local information with a central processor might be either unfeasible or not economical/efficient, owing to the large size of the network and volume of data, time-varying network topology, energy constraints, robustness and/or privacy concerns. Performing the computation in a centralized fashion may raise robustness concerns as well, since the central processor represents a bottleneck and an isolated point of failure.
For these reasons, effective learning methods must necessarily exploit distributed computation/learning architectures (with possibly parallelized multi-core processors), while keeping into account the distributed large-scale storage of data over the network and communication constraints. Very often, the implementation of such learning schemes requires the training of a \textit{shared} predictive function, i.e., a common model accessible independently by each of them. Considering the previous example, suppose that a set of embedded cameras is taking multiple high-resolution photos of a possible security threat. In this case, if the threat needs to be recognized quickly in the near future, the sensors have to train a shared classifier that must leverage on all the currently acquired photos, in order to obtain a sufficiently high accuracy. These problems are ubiquitous in the real world, and appear in many practical systems such as, e.g., wireless sensor networks \citep{predd2006distributed}, smart grids, distributed databases \citep{lazarevic2002boosting}, robotic swarms, just to name a few.

If a predictive behavior is needed, however, the designer of the distributed system has to answer a necessary question: What kind of model should be chosen as a classifier/regressor? Since deep NNs are currently obtaining state-of-the-art results in several fields \citep{schmidhuber2015deep,lecun2015deep}, employing them appears as a reasonable choice. Nevertheless, somewhat surprisingly, the literature on distributed training algorithms for NNs satisfying all the above requirements is extremely scarce. Most authors resort either to an ensemble of models trained independently by the agents \citep{lazarevic2002boosting,zhang2013privacy}, or to strategies requiring the sum of the gradients' contributions for all agents at every single iteration \citep{samet2012privacy,georgopoulos2014distributed}, exploiting the additivity of the gradients updates. Both these approaches can be easily shown to be unsatisfactory in general. In the former case, we have no guarantee that the ensemble of models will perform as good as a single model trained on the collection of all local datasets. In the latter case, instead, a global sum at every iteration might be infeasible due to an excessive amount of communication, particularly for large models comprising several hundred thousands parameters. It is also worth mentioning that a lot of research has been devoted recently to the design of parallel, asynchronous versions of stochastic gradient descent for training NNs on large clusters of commodity hardware \citep{dean2012large,sak2014squence,abadi2016tensorflow}. However, all these previous methods require the presence of at least one central server node, which coordinates the learning process; thus, they are not applicable in our context.

One of the reasons for the lack of distributed training methods for NNs is that, in principle, these methods require the solution of a distributed \textit{non-convex} optimization problem, which was tackled only in a few papers even in the optimization literature \citep{bianchi2013convergence,di2016next}. On the other side, if we turn our attention to methods for convex learning problems, the literature on their distributed training is vast, including algorithms for decentralized optimization of linear predictors \citep{xiao2007distributed,sayed2014adaptive,sayed2014adaptation}, sparse linear models \citep{mateos2010distributed,di2013sparse}, kernel ridge regression \citep{predd2006distributed,predd2009collaborative}, random-weights networks \citep{huang2015distributed,scardapane2015distributed,scardapane2016decentralized}, support vector machines \citep{navia2006distributed,lu2008distributed,forero2010consensus,scardapane2016distributed}, and kernel filtering \citep{perez2010robust,gao2015diffusion}.

\textit{Contribution:} In this paper, we propose an algorithmic framework for training general NN models in a fully distributed scenario, which encompasses several common loss functions and regularization terms.\footnote{A preliminary version of this work, focusing only on the squared loss function, was presented in \citep{di2016neuralnetworks}.} In particular, we build upon the in-network nonconvex optimization (NEXT) algorithm proposed in \citep{di2016next}, and recently extended in \citep{sun2016distributed} to handle general time-varying topologies. NEXT is one of the first methods to solve distributed non-convex optimization problems over networks of agents. The algorithm, which leverages on the so-called successive convex approximation (SCA) family of methods \citep{facchinei2015parallel}, is built upon two foundational ideas. First and foremost, at every iteration, the original non-convex problem is replaced with a strongly convex approximation, which is solved \textit{locally} at every agent. As we will illustrate along the paper, several kinds of convexification are possible, resulting in different trade-offs in terms of computational complexity and speed of convergence.
Second, the framework exploits a dynamic consensus procedure \citep{zhu2010discrete}, so that each agent can recover the information relative to all the other agents, which typically is not available at its local side. The resulting algorithms are shown to be convergent to a stationary solution of the social non-convex problem under loose requirements relative to the agents' communication topology, the choice of the algorithm's parameters, and the structure of the optimization problem. A further interesting aspect of the framework presented here is that the local optimization problems can be easily parallelized in a principled way (up to one NN parameter per available processor), without loosing the convergence properties of the framework. Consider, for example, the case of multiple medical institutions requiring the training of a common NN (e.g., for diagnosis purposes) leveraging on all historical clinical information \citep{vieira2006secure}. In this case, a decentralized algorithm is required due to strong privacy concerns on the release of medical, sensible information about the patients. Nonetheless, each institution may have access to an internal private cloud infrastructure. Using the framework outlined in this paper, privacy is guaranteed via the use of a distributed protocol, while each institution can parallelize its optimization steps using local cloud computing hardware. In this way, the resulting algorithms are both distributed (across the nodes) and parallel (inside each node) in nature. At the end of the (distributed) training process, each agent has access to the optimal set of NN's parameters, and it can apply the resulting model to newly arriving data (e.g., new photos taken from the camera) independently of the other agents. A comprehensive set of experimental results validate the proposed approach.

\textit{Outline of the paper:} The rest of the paper is organized as follows. In Section \ref{sec:problem_formulation}, we formalize the problem of distributed NN training. In Section \ref{sec:next} we describe the general framework for distributed NN training built upon the NEXT algorithm. Then, in Section \ref{sec:practical}, we consider the customization of the framework to different loss functions (squared loss, cross entropy, etc.) and regularization terms ($\ell_2$ norm, sparsity inducing penalties, etc.). Section \ref{sec:parallelizing_surrogate_optimization} describes a principled way to parallelize the optimization phase. In Section \ref{sec:experimental_validation}, we perform a large set of experiments aimed at assessing the performance of the proposed framework. Finally, Section \ref{sec:conclusion} draws some conclusions and future lines of research.

\textit{Notation:} We denote vectors using boldface lowercase letters, e.g., $\vect{a}$; matrices are denoted by boldface uppercase letters, e.g., $\vect{A}$. All vectors are assumed to be column vectors. The operator $\norm[p]{\cdot}$ is the standard $\ell_p$ norm on an Euclidean space. For $p=2$, it coincides with the Euclidean norm, while for $p=1$ we obtain the Manhattan (or taxicab) norm defined for a generic vector $\vect{v} \in \R^B$ as $\norm[1]{\vect{v}} = \sum_{k=1}^B |v_k|$. The notation $a[n]$ denotes the dependence of $a$ on the time-index $n$. Other notation is introduced along the paper when required.

\section{Problem Formulation}
\label{sec:problem_formulation}

Let us consider the problem of training a generic NN model $f(\vect{w}; \vect{x})$, where $\vect{x} \in \R^d$ denotes the $d$-dimensional input vector of the network, whereas $\vect{w} \in \R^Q$ is the vector collecting all the adaptable parameters that we aim to optimize. Note that we are considering the NN as a function of its parameters, as this will make the following derivation simpler. We are not concerned with the specific structure of the NN $f(\cdot)$ (i.e., number of hidden layers, choice of the activation functions, etc.), as long as the following assumptions are satisfied for any possible input vector $\vect{x} \in \R^d$.

\vspace{0.5em}
\noindent \textbf{Assumption A [On the NN model]:}
\begin{description}
\item[(A1)] $f$ is in $C^1$, i.e., it is continuously differentiable with respect to $\vect{w}$;\smallskip
\item[(A2)] $f$ has Lipschitz continuous gradient, with respect to $\vect{w}$, for some Lipschitz constant $L$, i.e.:
\begin{equation}
\left\lVert\nabla_{\vect{w}} f(\vect{w}_1; \vect{x}) - \nabla_{\vect{w}}f(\vect{w}_2; \vect{x})\right\rVert_2 \le L \bigl\lVert\vect{w}_1 - \vect{w}_2\bigr\rVert_2 \,.
\label{eq:lip_gradient}
\end{equation}
\end{description}
Assumption A is satisfied by most NN models commonly used in the literature, with the only notable exception of NN having non-differentiable activation functions such as ReLu neurons \citep{glorot2011deep}, maxout neurons \citep{goodfellow2013maxout}, and a few others. Nonetheless, convergence guarantees for these architectures are relatively uncommon even in the centralized case. In this paper, we are concerned with distributed architectures, where the data required to train the NN is not available on a centralized location, but is instead partitioned among $I$ interconnected agents. Prototypical examples of agents can be sensors in a wireless sensor networks (WSN), peers in a P2P network, power units in a smart grid, or mobile robots in a robotic swarm. At every specific time instant $n$, the communication network enabling interaction among the agents is modeled as a directed graph (digraph) $\mathcal{G}[n]=(\mathcal{V,E}[n])$, where $\mathcal{V}=\{1,\ldots,I\}$ is the vertex  set (i.e., the set of agents), and $\mathcal{E}[n]$ is the set of (possibly) time-varying directed edges. The in-neighborhood of agent $i$ at time $n$ (including node $i$) is defined as  $\mathcal{N}_i^{\rm in}[n]=\{j|(j,i)\in\mathcal{E}[n]\}\cup\{i\}$: node $i$ can receive information from node $j\neq i$ at time instant $n$ only if $j\in\mathcal{N}_i^{\rm in}[n]$. By assuming only single-hop communication, the resulting framework can be applied to the broadest possible class of problems.\footnote{More in general, $\mathcal{G}[n]$ corresponds to all feasible communication links between two agents. A multi-hop network can be described with an equivalent single-hop network by considering all possible paths as a direct link in the equivalent graph.} Due to this, each agent has a limited view and knowledge about the overall (possibly time-varying) network. Also, we assume that there is no agent (or finite number of them) that is able to collect all the data and coordinate the overall learning process. Associated with each graph $\mathcal{G}[n]$, we introduce (possibly) time-varying weights $c_{ij}[n]$ matching $\mathcal{G}[n]$:
\begin{align}\label{weights}
c_{ij}[n]=\left\{
             \begin{array}{ll}
               \theta_{ij}\in[\vartheta,1] & \hbox{if $j\in \mathcal{N}_i^{\rm in}[n]$;} \vspace{.2cm}\\
                 0 & \hbox{otherwise,}
             \end{array}
           \right.
\end{align}
for some $\vartheta\in (0,1)$, and define the matrix $\vect{C}[n]\triangleq (c_{ij}[n])_{i,j=1}^I$. These weights are used in the definition of the proposed algorithm in order to locally combine the information diffused over every neighborhood, i.e., $c_{ij}$ represents the weight given by agent $i$ to the information coming from agent $j$. The weights are given and are required to respect some properties listed later on in \eqref{double_stochastic}. Many choices are possible, and a brief overview can be found in \citet{di2016next}. Clearly, different setups for the weights may influence the convergence speed. Roughly speaking, simple choices like the one we detail in Section \ref{sec:experimental_setup} can be implemented immediately with no knowledge of the graph topology far from each neighborhood. On the contrary, more sophisticated weights can speedup convergence, while requiring global knowledge of the network and/or the solution to some optimization problem, e.g. see the strategies detailed in \citet{xiao2004fast}.

For the purpose of training the NN, we assume that the $i$th agent has access to a local training dataset of $N_i$ examples, denoted as $\mathcal{S}_i = \left\{ \vect{x}_{i,m}, d_{i,m} \right\}_{m=1}^{N_i}$, where we consider a single-output problem with $d_{i,m} \in \R$ for simplicity of overall notation. The output of the NN is an integer or a real value, depending on whether we are facing a classification task or a regression task, respectively. Given all the previous definitions, a general formulation for the distributed training of NNs can be cast as the minimization of a social cost function $G$ plus a regularization term $r(\cdot)$, which writes as:
\begin{equation}
\underset{\vect{w}}{\min} \;\; U(\vect{w}) = G(\vect{w}) + r(\vect{w}) = \sum_{i=1}^I g_i(\vect{w}) + r(\vect{w}) \,,
\label{Dist_NN_training}
\end{equation}
where $g_i(\cdot)$ is the error term relative to the $i$th local dataset:
\begin{equation}
g_i(\vect{w}) = \sum_{m\in \mathcal{S}_i} l\Bigl(d_{i,m}, f(\vect{w};\vect{x}_{i,m})\Bigr) \,,
\label{eq:local_NN_cost}
\end{equation}
\vspace{-.1cm}
\noindent with $l(\cdot, \cdot)$ denoting a generic (convex) loss function, while $r(\vect{w})$ is a regularization term. Due to the nonlinearity of the NN model $f(\vect{w};\vect{x})$, problem (\ref{Dist_NN_training}) is typically \textit{non-convex}. In this work, we consider the following assumptions on the functions involved in (\ref{Dist_NN_training})-(\ref{eq:local_NN_cost}).

\vspace{0.5em}
\noindent \textbf{Assumption B [On Problem (\ref{Dist_NN_training})]:}
\begin{description}
\item[(B1)] $l$ is convex and $C^1$, with Lipschitz continuous gradient; \smallskip
\item[(B2)] {$r$ satisfies (B1), or it is a nondifferentiable convex function with bounded subgradients;}\smallskip
\item[(B3)] {$U$ is coercive, i.e., $\displaystyle\lim_{\|\vect{w}\|\rightarrow \infty} U(\vect{w})=+\infty$.}
\end{description}

\noindent The structure of the function $l$ in (\ref{eq:local_NN_cost}) depends on the learning task (i.e., regression, classification, etc.). Typical choices are the squared loss for regression problems, and the cross-entropy for classification tasks \citep{haykin2009neural}. The regularization function $r(\vect{w})$ in (\ref{Dist_NN_training}) is commonly chosen to avoid overfitted solutions and/or impose a specific structure in the solution, e.g., sparsity or group sparsity. Typical choices are the $\ell_2$ and $\ell_1$ norms. All these functions satisfy Assumption B, and will be discussed in detail in the sequel. In view of the distributed nature of the problem, the $i$th agent knows its own cost function $g_i$ and the common regularization term $r$, but it does not have access to $g_j$ for $j \neq i$, nor can it exchange freely its own dataset $\mathcal{S}_i$ due to a variety of reasons, including privacy, data volume, and communication constraints. This aspect, combined with the non-convexity of $\eqref{Dist_NN_training}$, makes optimizing $\eqref{Dist_NN_training}$ in a distributed fashion a challenging problem, which has no ready-to-use solution available in the literature. The design of such algorithmic framework is the topic of the next three sections.

\section{NEXT: In-Network Successive Convex Approximation}
\label{sec:next}

In this section, we review the basics of the NEXT framework proposed in \citep{di2016next}, which was designed to solve general nonconvex distributed problems of the form \eqref{Dist_NN_training}. The next section will then focus on how to customize the framework to the NN distributed training problem considered in this paper. Due to lack of space, we provide only a very brief introduction to the NEXT framework, and we refer the interested readers to \citep{di2016next,sun2016distributed} for a full treatment, which also includes a proof of the convergence results. NEXT combines SCA techniques (Step 1) with dynamic consensus mechanisms (Steps 2 and 3), as described next.

\smallskip
\noindent \textbf{Step 1 (local SCA optimization):} Each agent $i$ maintains a local estimate $\vect{w}_i[n]$ of the optimization variable $\vect{w}$ that is iteratively updated. Solving directly Problem (\ref{Dist_NN_training}) may be too costly (due to the nonconvexity of $G$) and is  not even feasible in a distributed setting. One may then prefer to approximate Problem (\ref{Dist_NN_training}), in some suitable sense, in order to permit each agent to compute \emph{locally} and \emph{efficiently} the new iteration. In particular, writing  $G(\vect{w}_i)=g_i(\vect{w}_i)+\sum_{j\neq i}g_j(\vect{w}_i)$, we consider a convexification of  ${G}$ having the following form: i) at every iteration $n$, the (possibly) nonconvex  $g_i(\vect{w}_i)$ is replaced by a strongly convex surrogate, say  $\widetilde{g}_i(\cdot;\vect{w}_i[n]):\mathbb{R}^Q \rightarrow \mathbb{R}$, which may depend on the current iterate $\vect{w}_i[n]$; and  ii)  $\sum_{j\neq i}g_j(\vect{w}_i)$  is  linearized around $\vect{w}_i[n]$. More formally, the proposed  updating scheme reads: at every iteration $n$,  given the local  estimate $\vect{w}_i[n]$, each agent $i$ solves the \emph{strongly convex} optimization problem:\vspace{-0.1cm}
\begin{align}\label{best_resp_x_hat_2}
&\widetilde{\vect{w}}_i[n]\,=\,\underset{\vect{w}_i}{\arg\min} \;\;\widetilde{U}_i\left(\vect{w}_i;\vect{w}_i[n],\boldsymbol{\pi}_{i}[n]\right) \\ &=\,\underset{\vect{w}_i}{\arg\min} \;\; \widetilde{g}_i(\vect{w}_i;\vect{w}_i[n])+\boldsymbol{\pi}_{i}[n]^{T}(\vect{w}_i-\vect{w}_i[n])+r(\vect{w}_i), \nonumber
\end{align}
where
\begin{equation}\label{pi}
\boldsymbol{\pi}_i[n]\triangleq\sum_{j\neq i}\nabla_{\vect{w}}\;g_j(\vect{w}_i[n]). \vspace{-0.1cm}
\end{equation}
The evaluation of (\ref{pi}) would require the knowledge of all $\nabla g_j(\vect{w}_i[n])$, $j\neq i$ at node $i$. This information is not directly available at node $i$; we will cope with this local lack of global knowledge later on in step 3. Once the surrogate problem (\ref{best_resp_x_hat_2}) is solved, each agent computes an auxiliary variable, say $\vect{z}_i[n]$, as the convex combination:
\begin{equation}
\vect{z}_i[n] = \vect{w}_i[n] + \alpha[n]\left( \widetilde{\vect{w}}_i[n] - \vect{w}_i[n] \right) \,,
\label{eq:z_update_next}
\end{equation}
where $\alpha[n]$ is a possibly time-varying step-size sequence. This concludes the optimization phase of the algorithm. An appropriate choice of the surrogate function $\widetilde{g}_i(\cdot;\vect{w}_i[n])$ guarantees the coincidence between the fixed-points of  $\widetilde{\vect{w}}_i[n]$ and the stationary solutions of Problem (\ref{Dist_NN_training}). The main results are given in the following proposition \citep{facchinei2015parallel}:

\begin{proposition}
Given Problem (\ref{Dist_NN_training}) under A1-A2 and B1-B3, suppose that $\widetilde{g}_i$ satisfies the following conditions: \vspace{-.1cm}
\begin{description}
\item[ (F1)]
 $\widetilde{g}_{i} (\mathbf{\cdot};\vect{w})$ is uniformly strongly convex with $\tau_i>0$;\smallskip
\item[  (F2)]  $\nabla \widetilde{g}_{i} (\vect{w};\vect{w}) = \nabla g_i(\vect{w})$ for all $\vect{w}$;\smallskip
\item[  (F3)]  $\nabla \widetilde{g}_{i} (\vect{w};\mathbf{\cdot})$ is uniformly Lipschitz continuous.
\end{description}
Then, the set of fixed-point of $\widetilde{\vect{w}}_i[n]$ in (\ref{best_resp_x_hat_2}) coincides with that of the stationary solutions of  (\ref{Dist_NN_training}).
\end{proposition}
\noindent Conditions F1-F3 state that $\widetilde{g}_{i}$ should be regarded as a strongly convex approximation of $g_i$ at the point $\vect{w}$, which preserves the first order properties of $g_i$. Several feasible choices are possible for a given $g_i$; the appropriate one depends on computational and communication requirements. The goal of the next section will be to illustrate some possible choices for the local surrogate cost $\widetilde{g}_{i}$ properly customized to our distributed NN training problem.

\smallskip
\noindent \textbf{Step 2 (agreement update):} To force the asymptotic agreement among the $\vect{w}_i$'s, a consensus-based step is employed on the auxiliary variables $\vect{z}_i[n]$'s. Each agent $i$ updates its local variable $\vect{w}_i[n]$ as:
 \begin{equation}\label{consensus_update}
\vect{w}_i[n+1]= \sum_{j\in \mathcal{N}_i^{\rm in}[n]} c_{ij}[n]\, \vect{z}_i[n],\vspace{-0.1cm}
\end{equation}
where $\vect{C}[n]=(c_{ij}[n])_{ij}$ is defined in (\ref{weights}), and satisfies
\begin{align}\label{double_stochastic}
\vect{C}[n]\,\mathbf{1}=\mathbf{1} \quad \text{and}\quad \mathbf{1}^T \vect{C}[n]=\mathbf{1}^T \quad \forall n.
\end{align}
Since the weights are constrained by the network topology, \eqref{consensus_update} can be implemented via local message exchanges: agent $i$ updates its estimate $\vect{w}_i$ by averaging over the current solutions $\vect{z}_j[n]$ received from its neighbors. The double stochasticity condition in (\ref{double_stochastic}) can be achieved according to a variety of predefined strategies, including the Metropolis-Hastings criterion \citep{xiao2007distributed}, or by optimizing a cost function with respect to the spectral properties of the graph \citep{xiao2004fast}.

\smallskip
\noindent \textbf{Step 3 (diffusion of information over the network):} The computation of $\widetilde{\vect{w}}_{i}[n]$ in (\ref{best_resp_x_hat_2}) is not fully distributed yet, because the evaluation of
$\boldsymbol{\pi}_{i}[n]$ in (\ref{pi}) would require the knowledge of all $\nabla g_j(\vect{w}_i[n])$, $j\neq i$, which is a global information that is not available locally at node $i$.  To cope with this issue, as proposed in \citep{di2016next}, we replace $\boldsymbol{\pi}_i[n]$  in (\ref{best_resp_x_hat_2}) with a \emph{local} estimate, say  $\widetilde{\boldsymbol{\pi}}_{i}[n]$,  asymptotically converging to $\boldsymbol{\pi}_{i}[n]$. Thus, we can update the local estimate $\widetilde{\boldsymbol{\pi}}_{i}[n]$ in a fully distributed manner as:
\begin{equation}\label{pi3}
\widetilde{\boldsymbol{\pi}}_i[n]\triangleq I\cdot \vect{y}_i[n]-\nabla g_i(\vect{w}_i[n]),
\end{equation}
where $\vect{y}_i[n]$  is a local auxiliary variable (controlled by agent $i$) that aims to asymptotically track the average of the gradients. This can be done updating $\vect{y}_i[n]$ according to the following dynamic consensus recursion:\vspace{-0.2cm}
\begin{equation}\label{y2}
\hspace{-0.04cm}\vect{y}_i[n+1]\triangleq\sum_{j=1}^I c_{ij}[n]\vect{y}_j[n] \hspace{-0.02cm} + \hspace{-0.02cm}\left(\nabla g_i(\vect{w}_i[n+1])\hspace{-0.02cm}-\hspace{-0.02cm}\nabla g_i(\vect{w}_i[n])\right) \hspace{-0.2cm}\vspace{-0.05cm}
\end{equation}
where $\vect{y}_i[0]\triangleq\nabla_{\vect{w}_i}g_i(\vect{w}_i[0])$, and can be computed locally by every agent. Note that the update of $\vect{y}_i[n]$ and thus $\widetilde{\boldsymbol{\pi}}_i[n]$ can be now performed locally with message exchanges with the agents in the neighborhood.

\smallskip

The overall procedure is summarized in Algorithm \ref{alg:general}, where $\nabla \vect{g}_i[n]$ is used as a simplified notation for $\nabla_{\vect{w}_i} g_i(\vect{w}_i[n])$. Its convergence properties are reported in the following Proposition.

\begin{proposition}{
Let $\{\vect{w}[n]\}_n\triangleq \{(\vect{w}_i[n])_{i=1}^I\}_n$ be the sequence generated by Algorithm 1, and let $\{\overline{\vect{w}}[n]\}_n\triangleq \{(1/I)\,\sum_{i=1}^I\vect{w}_i[n]\}_n$ be its average. Suppose that i) Assumptions A and B hold; ii) the sequence of graphs describing the network is $B$-strongly connected\footnote{Formally, there exists an integer $B > 0$ such that the graph $\mathcal{G}[k]=(\mathcal{V},\mathcal{E}_B[k])$, with $\mathcal{E}_B[k]=\bigcup_{n=kB}^{(k+1)B-1}\mathcal{E}[n]$
is strongly connected, for all $k\geq0$.}; iii) condition (\ref{double_stochastic}) holds; and iv) the step-size sequence $\{\alpha[n]\}_n$ is chosen so that  $\alpha[n]\in (0,1]$ for all $n$ and $\sum_{n=0}^{\infty}\alpha[n]=\infty$.
%
Then, (a) all the limit points of the sequence $\{\overline{\vect{w}}[n]\}_n$ are stationary solutions of (\ref{Dist_NN_training}); (b)  all the sequences  $\{\vect{w}_i[n]\}_n$  asymptotically agree, i.e., $\|\vect{w}_{i}[n]-\overline{\vect{w}}[n]\|_2\underset{n\rightarrow\infty}{\longrightarrow}0
 $, for all $i$.}
 \label{convergence_th}
\end{proposition}
\begin{proof}
Algorithm 1 is a special case of an extension of the NEXT framework proposed in \citep{sun2016distributed} (i.e., the SONATA algorithm). Then, under the above assumptions on the NN model in (\ref{Dist_NN_training}), the network among agents, and the algorithm's parameters, all conditions of Theorem 1 in \citep{sun2016distributed} are satisfied, and the convergence result follows.
\end{proof}

\begin{algorithm}

\smallskip
$\textbf{Data}:$ $\vect{w}_{i}[0]$, $\vect{y}_i[0]= \nabla g_i[0]$,    $\boldsymbol{\pi}_{i}[0]=I\vect{y}_i[0]-\nabla g_i[0]$, $\forall i=1,\ldots ,I$, and $\{\vect{C}[n]\}_n$. Set $n=0$.\smallskip

\texttt{$\mbox{(S.1)}$}$\,\,$If $\mathbf{w}_i[{n}]$ satisfies a global termination
criterion: STOP;\smallskip

\texttt{$\mbox{(S.2)}$} \texttt{Local Optimization}:  Each agent $i$ \vspace{0.2cm}

\hspace{1.1cm} (a) computes $\widetilde{\vect{w}}_{i}[n]$ as:
\begin{align}\label{opt_prob_alg}
\widetilde{\vect{w}}_{i}[n]\,=&\,\underset{\vect{w}_{i}}{\arg\min} \;\;\widetilde{U}_{i}\left(\vect{w}_{i};\vect{w}_{i}[n],\widetilde{\boldsymbol{\pi}}_{i}[n]\right)  \,,
\end{align}

\hspace{1.1cm} \vspace{-0.05cm} (b) updates its local variable $\vect{z}_i[n]$:
\begin{equation}
\vect{z}_i[n]=\vect{w}_i[n]+\alpha[n]\left(\widetilde{\vect{w}}_{i}[n]-\vect{w}_i[n]\right) \,. \nonumber
\end{equation}

\vspace{-0.05cm}
\texttt{$\mbox{(S.3)}$} \texttt{Consensus update}:   Each agent $i$ \vspace{0.2cm}

  \hspace{1.1cm} (a) collects $\vect{z}_j[n]$ and $\vect{y}_j[n]$ from neighbors, \vspace{0.3cm}

  \hspace{1.1cm} (b) updates $\vect{w}_i[n]$ as:
  \begin{equation}
  \vect{w}_i[n+1] = \sum_{j=1}^I c_{ij}[n]\, \vect{z}_j[n] \nonumber \,,
  \end{equation}

  \hspace{1.1cm}(c) updates $\vect{y}_i[n]$ as:
  \begin{equation}
  \displaystyle \vect{y}_i[n+1]=\sum_{j=1}^I c_{ij}[n]\,\vect{y}_j[n]+\left(\nabla g_i[n+1]-\nabla g_i[n]\right) \nonumber\,,
  \end{equation}

  \hspace{1.1cm} (d) updates $\widetilde{\boldsymbol{\pi}}_{i}[n]$ \text{ as }:
  \begin{equation}
  \widetilde{\boldsymbol{\pi}}_{i}[n+1]=I\cdot \vect{y}_i[n+1]-\nabla g_i[n+1] \nonumber \,.
  \end{equation}

\vspace{0.3cm}
\texttt{$\mbox{(S.4)}$} $n\leftarrow n+1$, and go to \texttt{$\mbox{(S.1)}.$}

\protect\caption{\hspace{-2.5pt}\textbf{:} \label{alg:general}\textbf{NEXT Framework for Distributed Optimization of \eqref{Dist_NN_training}}}
\end{algorithm}

\noindent It is interesting to notice that convergence conditions are particularly loose. With respect to the network connecting the agents, it is enough to ensure connectivity over a finite (but arbitrary) \textit{union} of time instants. 

Step-size sequences satisfying the conditions can be derived easily, either fixed (and sufficiently small) as remarked in \citep{sun2016distributed}, or diminishing, e.g., using the following quadratically decreasing rule that was found particularly effective in our experiments:
\begin{equation}
\alpha[n] = \alpha[n-1]\left( 1 - \varepsilon\alpha[n-1] \right) \,,
\label{eq:step_size_sequence_example}
\end{equation}
where $\alpha[0], \varepsilon \in \left(0, 1\right]$  must be chosen by the user.

The per-iteration cost of the algorithm is clearly dominated by the solution of the surrogate optimization problem in (\ref{opt_prob_alg}). As we will see in the next section, the flexibility of the framework allows to select different choices of surrogate functions, typically impacting the complexity/performance tradeoff of the algorithm. The framework can be accelerated in two ways. First, we can parallelize the surrogate optimization in (\ref{opt_prob_alg}); this point will be discussed in Section \ref{sec:parallelizing_surrogate_optimization}. Second, at each iteration $n$, we can consider an \textit{inexact} solution of the surrogate problems in (\ref{opt_prob_alg}) within a user-specified error bound $\epsilon_i[n]$. In this case, it can be shown that convergence is still guaranteed, as long as the following condition is satisfied:
$\sum_{n=0}^{\infty} \alpha[n]\epsilon_i[n] < \infty$, $\forall i \in 1, \ldots, I$,
which establish a decaying rate of the error sequence over time. For further details, we refer to \citep[Theorem 4]{di2016next}.

\section{Strategies for Distributed NN Training}
\label{sec:practical}

In this section, we customize the NEXT framework for the solution of several distributed NN training problems. In particular, we focus on the choice of the surrogate functions $\widetilde{g}_i$ in (\ref{best_resp_x_hat_2}). From Proposition 1, we know that they must be chosen to satisfy F1-F3. Thus, we explore two general-purpose strategies that can be used to this end, before analyzing some practical algorithms resulting from the combination of these two strategies with common choices of the loss function and the regularization term. Essentially, the aim of $\widetilde{g}_i(\cdot)$ is to provide a strongly convex approximation of (the non-convex) $g_i$ around the current point, preserving (at least) the first-order information of the original function. Then, the most basic idea is to linearize the entire $g_i$, irrespectively of the actual choice of loss function $l$, as:
\begin{align}
\widetilde{g}_i^{\text{FL}}(\vect{w}_i; \vect{w}_i[n])  \;=\; & g_i(\vect{w}_i[n]) + \nabla g_i(\vect{w}_i[n])^T(\vect{w}_i - \vect{w}_i[n]) \, \nonumber \\ & +\frac{\tau}{2}\norm{\vect{w}_i - \vect{w}_i[n]}^2 \,,
\label{eq:full_linearization}
\end{align}
where the last term in \eqref{eq:full_linearization} is a proximal regularization term (with $\tau \geq 0$) used to ensure strong convexity; in what follows, we will refer to \eqref{eq:full_linearization} as the full linearization strategy (FL). In general, the use of the FL strategy leads to the formulation of surrogate problems in \eqref{opt_prob_alg} allowing for a simple, closed-form solution for most choices of regularization. At the same time, this strategy is throwing away most information of $g_i(\cdot)$, by only keeping first-order information on its gradient. For this reason, the resulting family of algorithms can possess a slow convergence speed, similarly to what happens with the use of (centralized) steepest descent optimization procedures.

To implement a more sophisticated approximation aimed at preserving the hidden convexity in the problem, we start noticing that the loss function in \eqref{eq:local_NN_cost} is composed of the summation of terms, each one given by the composition of an exterior convex function (i.e., the loss function $l$), and an interior nonlinear function (i.e., the NN model $f$). Then, a possible choice for $\widetilde{g}_i$ is to preserve the convexity of $l$, while linearizing $f$ around the current estimate $\vect{w}_i[n]$, and a generic input point $\vect{x}_{i,m}$, as:
\begin{align}
\widetilde{f}(\vect{w}_i;\vect{w}_i[n],\vect{x}_{i,m}) = f(\vect{w}_i[n]; \vect{x}_{i,m})  +  \nabla f(\vect{w}_i[n];\vect{x}_{i,m})^T(\vect{w}_i-\vect{w}_i[n]) \,.
\label{eq:f_tilde}
\end{align}
Then, the surrogate $\widetilde{g}_i$ is obtained as:
\begin{align}
 \widetilde{g}_i^{\text{PL}}(\vect{w}_i; \vect{w}_i[n]) =\;  \sum_{m\in \mathcal{S}_i} l(d_{i,m}, \widetilde{f}(\vect{w}_i;\vect{w}_i[n],\vect{x}_{i,m}))  +\dfrac{\tau}{2} \|\vect{w}_{i} - \vect{w}_{i}[n]\|^2,
 \label{eq:partial_linearization}
\end{align}
with $\tau\geq0$. We will refer to (\ref{eq:partial_linearization}) as the partial linearization (PL) strategy. It is straightforward to check that the surrogate $\widetilde{g}_i^{\text{PL}}$ in (\ref{eq:partial_linearization}) satisfies the properties F1-F3 required by Proposition 1.

In the remainder of the section, we consider a set of practical examples resulting from the use of our general framework.



\subsection{Case 1: ridge regression cost}
\label{sec:ridge_regression}

As a first example, we consider the use of a squared loss function combined with a classical $\ell_2$ norm regularization on the weights (also known as weight decay in the NN literature \citep{moody1995simple}):
\begin{equation}
l(a, b) \triangleq \left( a - b \right)^2 , \;\; r(\vect{w}) \triangleq \frac{\lambda}{2} \norm{\vect{w}}^2 \,,
\label{eq:ridge_regression}
\end{equation}
where $\lambda$ is a positive regularization parameter. Historically, this is the most common training criterion for NNs, and it is still widely used today for regression problems. Being equivalent to a nonlinear ridge regression, we borrow this terminology here.

Let us begin with the FL strategy in (\ref{eq:full_linearization}). Note that, thanks to the specific form of the regularizer, the resulting optimization problem in (\ref{opt_prob_alg}) is already strongly convex, so that we can set $\tau=0$. Then, using (\ref{eq:full_linearization}) and (\ref{eq:ridge_regression}), the surrogate problem in (\ref{opt_prob_alg}) reduces to the minimization of a positive definite quadratic function, which admits a simple closed form solution, given by:
\begin{empheq}[box=\widefbox]{equation}
\widetilde{\vect{w}}_i[n] = -\frac{1}{\lambda}\Bigl( \nabla \vect{g}_i[n] + \widetilde{\boldsymbol{\pi}}_i[n] \Bigr) \,,
\label{eq:ridge_FL}
\end{empheq}
where as before $\nabla \vect{g}_i[n]$ is used as a simplified notation for $\nabla_{\vect{w}_i} g_i(\vect{w}_i[n])$. Eq. \eqref{eq:ridge_FL} represents the first practical implementation of the framework in Algorithm 1 for distributed NN training. As we can notice from \eqref{eq:ridge_FL}, the FL strategy discards all information on the global cost function $U$ in (\ref{Dist_NN_training}), except for a first-order approximation. Thus, the descent direction in \eqref{eq:ridge_FL} will be proportional to the opposite of the gradient of $U$, thanks to the current estimate $\widetilde{\boldsymbol{\pi}}_i[n]$ of \eqref{pi} that is locally available at node $i$. As we will see in the numerical results, the performance of the resulting distributed scheme is similar to a centralized gradient method, sharing its advantages (low computational complexity) and its drawbacks (possible slow convergence speed).

We now proceed considering the PL strategy in (\ref{eq:partial_linearization}). To this aim, let us introduce the following `residual' terms:
\begin{equation}
r_{i,m}[n] = d_{i,m} - f(\vect{w}_i[n]; \vect{x}_{i,m}) + \vect{J}_{i,m}[n]^T \vect{w}_i[n] \,,
\label{eq:residual}
\end{equation}
where $\vect{J}_{i,m}[n] = \nabla_{\vect{w}_i} f(\vect{w}_i[n];\vect{x}_{i,m})$ is a $Q$-dimensional vector containing the derivatives of the NN output with respect to any single weight parameter. In the general case, it will be a matrix with one column per NN output. This quantity is sometimes denoted as the weight Jacobian \citep{blackwell2012neural}, since it measures the influence of a small parameter change on the output of the neural network.\footnote{Note that a single back-propagation step per iteration is needed to build the weight Jacobian, as discussed in \citet[Section 5.3.4]{bishop2006pattern}.} Now, using (\ref{eq:partial_linearization}) in (\ref{opt_prob_alg}), it is easy to show that the surrogate problem can be written again as the minimization of a positive definite quadratic form, given by:
\begin{equation}
\widetilde{\vect{w}}_i[n]  =  \underset{\vect{w}_i}{\arg\min}\;\; \vect{w}_i^T \Bigl( \vect{A}_i[n] + \frac{\lambda}{2}\vect{I} \Bigr)\vect{w}_i - 2\vect{b}_i[n]^T\vect{w}_i \,,
\label{eq:surrogate_ridge_pl}
\end{equation}
where \vspace{-.5cm}
\begin{align}
\vect{A}_i[n] &= \sum_{m\in S_i} \vect{J}_{i,m}[n]\vect{J}_{i,m}[n]^T \,, \label{eq:Ai}\\
\vect{b}_i[n] &= \sum_{m\in S_i} \vect{J}_{i,m}[n]r_{i,m}[n] - 0.5\,\widetilde{\boldsymbol{\pi}}_i[n] \label{eq:bi}\,.
\end{align}
As an interesting side note, in the NN literature the matrix \eqref{eq:Ai} is known as an outer product approximation to the Hessian matrix of $g_i(\cdot)$ (i.e., the error function local to agent $i$), which is obtained by assuming that the error is uncorrelated with the second derivative of the network's output \citep[Section 5.4.2]{bishop2006pattern}. Finally, solving the resulting minimization problem in (\ref{eq:surrogate_ridge_pl}), the solution $\widetilde{\vect{w}}_i[n]$ of the surrogate problem, to be used in (\ref{opt_prob_alg}), is given by:
\begin{empheq}[box=\widefbox]{equation}
\widetilde{\vect{w}}_i[n] = \Bigl( \vect{A}_i[n] + \frac{\lambda}{2}\vect{I} \Bigr)^{-1}\vect{b}_i[n] \,.
\label{eq:ridge_PL}
\end{empheq}
Differently from the FL strategy, whose computational complexity is linear in the number of parameters, in this case solving the surrogate problem is of the order $\mathcal{O}(Q^3)$, where $Q$ is the number of adaptable NN parameters, due to the matrix inversion step. Nevertheless, as we will see in the numerical results, the resulting descent direction provides a very large improvement in terms of convergence speed. Additionally, this strategy can benefit from a larger relative speedup when employing the parallelization strategy described in Section \ref{sec:parallelizing_surrogate_optimization}.

\subsection{Case 2a: squared error with weight sparsity}
\label{sec:lasso}

As a second example, let us consider again the use of a squared loss term $l$ in (\ref{eq:ridge_regression}), combined this time with a sparsity promoting term given by the $\ell_1$ norm on the weight vector, i.e.,
\begin{equation}\label{eq:ell1}
r(\vect{w}) \triangleq \lambda \norm[1]{\vect{w}} = \lambda \sum_{k=1}^Q | w_k | \,.
\end{equation}
The $\ell_1$ norm promotes sparsity of the weight vector, acting as a convex approximation of the non-convex, non-differentiable $\ell_0$ norm \citep{tibshirani1996regression}. While there exists customized algorithms to solve non-convex $\ell_1$ regularized problems \citep{ochs2015iteratively}, it is common in the NN literature to apply first-order procedures (e.g., stochastic descent with momentum) followed by a thresholding step to obtain sparse solutions \citep{bengio2012practical}. In what follows, we illustrate the customization of the NEXT framework to this use case, using both FL and PL strategies in (\ref{eq:full_linearization}) and (\ref{eq:partial_linearization}), respectively. In the FL case, using (\ref{eq:full_linearization}) and (\ref{eq:ell1}), with $\tau>0$ to ensure strong convexity, the problem in (\ref{opt_prob_alg}) can be written as the minimization of the sum of $q$ independent functions, as follows:
\begin{align}
\widetilde{\vect{w}}_i[n]  =\, & \underset{\vect{w}_i}{\arg\min}\; \sum_{k=1}^q \Bigl\{ \left( \nabla g_{ik}[n] + \widetilde{\pi}_{ik}[n] - \tau w_{ik}[n] \right) w_{ik} \Bigr. \,   \nonumber \\ & \Bigl.  +\frac{\tau}{2} w_{ik}^2 + \lambda |w_{ik}| \Bigr\} \,.
\label{eq:surrogate_lasso_fl}
\end{align}
After some easy calculations, the solution of the optimization problem in \eqref{eq:surrogate_lasso_fl} is given by the closed form expression:
\begin{empheq}[box=\widefbox]{equation}
\widetilde{\vect{w}}_i[n] = \mathcal{S}_{\lambda/\tau} \left( \vect{w}_{i}[n] - \frac{1}{\tau} \nabla \, \vect{g}_i[n]  - \frac{1}{\tau}\widetilde{\boldsymbol{\pi}}_i[n] \right) \,,
\label{eq:minimizer_surrogate_LASSO_parallel2}
\end{empheq}
where \vspace{-.4cm}
\begin{align}
\mathcal{S}_{\gamma}(z)={\rm sign}(z)\max(0,|z|-\gamma),
\label{eq:soft_thresh}
\end{align}
is the (component-wise) soft thresholding function.

%

In the PL case, using (\ref{eq:partial_linearization}) and (\ref{eq:ell1}), the problem in (\ref{opt_prob_alg}) can be cast as an $\ell_1$ regularized quadratic program, given by:
\begin{empheq}[box=\widefbox]{align}
\widetilde{\vect{w}}_i[n] & = \underset{\vect{w}_i}{\arg\min} \Biggl\{ \vect{w}_i^T \left( \vect{A}_i[n] + \frac{\tau}{2}\vect{I} \right) \vect{w}_i \,- \Biggr. \nonumber \\  & \Biggl.  2\left( \vect{b}_i[n] + 0.5\tau\vect{w}_i[n] \right)^T \vect{w}_i + \lambda \norm[1]{\vect{w}_i} \Biggr\} \,,
\label{eq:surrogate_lasso_pl}
\end{empheq}
where $\vect{A}_i[n]$ and $\vect{b}_i[n]$ are given by (\ref{eq:Ai}) and (\ref{eq:bi}), respectively.
This is the first case we encounter where the solution of the optimization step cannot be expressed immediately in a closed form. Nevertheless, problem \eqref{eq:surrogate_lasso_pl} is given by the sum of a strongly convex function and an $\ell_1$ term, and many efficient strategies can be used for its approximate solution, including FISTA \citep{beck2009fast}, coordinate descent methods \citep{cevher2014convex}, and several others.

\subsection{Case 2b: group sparse penalization}
\label{sec:group_sparse}

The formulation introduced in Sec. \ref{sec:lasso} can be easily extended to handle a \textit{group sparse} penalization, which allows the selective removal of entire neurons during the training process, see, e.g., \citep{scardapane2017group}. The basic idea is to force all the outgoing weights from a neuron to be \textit{simultaneously} either non-zero or zero; the latter resulting in the direct removal of the neuron itself. Note that a neuron here can correspond to an input neuron, to a neuron in a hidden layer, or to a bias term, thus allowing the removal of input features, hidden neurons, and bias terms from the trained network (see \citep{scardapane2017group} for details). To this aim, let us suppose that the neurons are ordered and indexed as $1, \ldots, P$. Also, let us denote by $\vect{w}_{i,p}, \, p=1,\ldots,P$, the subset of weights of $\vect{w}_i$ collecting all connections between the $p$th neuron with all the neurons in the following layer. Group sparsity can then be imposed by choosing in (\ref{Dist_NN_training}) the following regularization term:
\begin{equation}
r(\vect{w}) \triangleq \lambda \displaystyle \sum_{p=1}^P \rho_p \norm{\vect{w}_{p}} \,,
\label{eq:reg_group_lasso}
\end{equation}
where $\rho_p = \sqrt{r_p}$ are positive constants, $p=1,\ldots,P$, with $r_p$ denoting the dimensionality of $\vect{w}_p$.

Let us now analyze the customization of our framework when the FL strategy in (\ref{eq:full_linearization}) is applied. Then, let us define $\vect{a}_i[n] = \nabla \vect{g}_{i}[n] + \widetilde{\boldsymbol{\pi}}_{i}[n] - \tau \vect{w}_{i}[n]$, denoting with $\vect{a}_{i,p}$ the restriction of $\vect{a}_i$ to the indexes associated with the $p$th group. Thus, the surrogate problem in (\ref{opt_prob_alg}) writes as:
\begin{empheq}[box=\widefbox]{align}
\widetilde{\vect{w}}_i[n] = \underset{\vect{w}_i}{\arg\min} \sum_{p=1}^P \Biggl\{ \vect{a}_{i,p}^T\vect{w}_{i,p} + \Biggr. \Biggl. \frac{\tau}{2}\norm{\vect{w}_{i,p}}^2 + \lambda\rho_p \norm{\vect{w}_{i,p}} \Biggr\} \,.
\label{eq:group_lasso_fl}
\end{empheq}
As we can notice from (\ref{eq:group_lasso_fl}), the cost function is given by a summation of costs, each one dependent on a single neuron. Also in this case, even if problem (\ref{eq:group_lasso_fl}) cannot be solved in closed form, it is possible to implement very fast and efficient algorithms for its solution, see, e.g., \citep{schmidt2010graphical,cevher2014convex}. Furthermore, in the case each agent has a multi-core architecture, the structure of (\ref{eq:group_lasso_fl}) makes straightforward the parallelization of computation, where each local processor can take care only of a subset of neurons. Finally, considering the PL strategy, the resulting formulation is equivalent to \eqref{eq:surrogate_lasso_pl}, with the only difference that \eqref{eq:reg_group_lasso} replaces the $\ell_1$ norm. Again, many of the techniques mentioned before can be used to solve also the resulting (group sparse) strongly convex problem.

%
%

\subsection{Case 3: cross-entropy loss}
\label{sec:cross_entropy}

As an additional example, let us consider the case of binary classification, i.e., $d_{i,m} \in \left\{0,1\right\}$. Then, assuming the output of the NN is limited between $0$ and $1$, a standard optimization criterion involves the cross-entropy loss function in (\ref{Dist_NN_training}), i.e.:
\begin{equation}
l(a,b) \triangleq - a\log(b) - (1-a)\log(1-b) \,.
\label{eq:cross_entropy_loss}
\end{equation}
In this case, using the FL strategy in \eqref{eq:full_linearization}, we obtain the same closed form solution as in
\eqref{eq:ridge_FL} (or \eqref{eq:surrogate_lasso_fl}) by using the $\ell_2$ (or $\ell_1$) regularization, with the only difference being that each function $g_i$ in (\ref{eq:local_NN_cost}) now depends on the cross-entropy loss in \eqref{eq:cross_entropy_loss}. The PL case, instead, requires some additional care. In particular, although the NN output is bounded, the same is not true for its linear approximation \eqref{eq:f_tilde}. Simply substituting \eqref{eq:f_tilde} in \eqref{eq:cross_entropy_loss} might result in undefined values, since the argument of the logarithm must be positive. To tackle this issue, let us notice that in this case the NN model can be written as:
\begin{equation}
f(\vect{w}; \vect{x}) = \sigma\left(f_L(\vect{w}; \vect{x})\right) \,,
\end{equation}
where $\sigma(\cdot)$ is a squashing function (without loss of generality, we assume it to be a sigmoid), and $f_L$ is the NN output up to (but not including) the activation function of the output neuron. The sigmoid $\sigma(z)$ is non-convex, but its internal composition with the cross-entropy loss in \eqref{eq:cross_entropy_loss} is convex, see, e.g., \citep{boyd2004convex}. Exploiting such hidden convexity, we can write the surrogate problem in (\ref{opt_prob_alg}), while satisfying the conditions F1-F3 in Proposition 1, as follows:
\begin{empheq}[box=\widefbox]{align}
\widetilde{\vect{w}_i}[n] \,= \,& \underset{\vect{w}_i}{\arg\min}  \Biggl\{  l\left(d_{i,m}, \sigma\left(\widetilde{f}_L(\vect{w}_i;\vect{w}_i[n],\vect{x}_{i,m})\right)\right) \, \Biggr. \nonumber \\ & \Biggl. +\boldsymbol{\pi}_i[n]^T \vect{w}_i + r(\vect{w}_i) + \frac{\tau}{2} \norm{ \vect{w}_i - \vect{w}_i[n] }^2 \Biggr\} \,,
\label{eq:surrogate_cross_entropy_pl}
\end{empheq}
where $\widetilde{f}_L(\cdot)$ is the first-order linearization of $f_L$ defined as in \eqref{eq:f_tilde}. Also in this case we cannot make any further simplifications although, once again, the strong convexity of the problem makes it relatively easy to be solved (roughly equivalent to a traditional logistic regression).


\section{Parallelizing the Local Optimization}
\label{sec:parallelizing_surrogate_optimization}

In this section, we explore how each agent can parallelize the local optimization in (\ref{opt_prob_alg}), when having access to $C$ separate computing machines (e.g., cores, or computers in a cloud). As we stated in the introduction, this effectively gives rise to algorithms that are both distributed (across agents) and parallel (inside each agent) in nature. To this end, suppose that the weight vector $\vect{w}_i$ is partitioned in $C$ non-overlapping blocks $\vect{w}_{i,1}, \ldots, \vect{w}_{i,C}$, so that $\vect{w}_i = \bigcup_{c=1}^C \vect{w}_{i,c}$ (assuming that the union keeps the original order). Note that we use a similar notation as in Section \ref{sec:group_sparse} to identify a single group, i.e., using an additional subscript under the variable. For convenience, we also define $\vect{w}_{i,-c}[n]\triangleq (\vect{w}_{i,p}[n])_{1=p\neq c}^C$ as the tuple of all blocks excepts the $c$-th one, and similarly for all other variables. Additionally, we assume that the regularization term $r$ is block separable, i.e., $r(\vect{w}_i)=\sum_{c=1}^C r_{i,c}(\vect{w}_{i,c})$ for some $r_{i,c}$. This is true for the $\ell_2$ and $\ell_1$ norms, and it holds true also for the group sparse norm in (\ref{eq:reg_group_lasso}) if we choose the groups in a consistent way. Then, the key idea is to decompose (\ref{opt_prob_alg}) on a per-core-basis, and solve a sequence of (strongly) convex low-complexity subproblems, whereby all processors of agent $i$ update their blocks in parallel. To this aim, we build a surrogate function $\widetilde{g}_i$ that additively decomposes over the different cores, i.e.:
\begin{equation}
\widetilde{g}_i(\vect{w}_i; \vect{w}_i[n])=\sum_{c=1}^C \widetilde{g}_{i,c}(\vect{w}_{i,c};\vect{w}_{i,-c}[n]) \,,
\label{eq:parallel_surrogate_function}
\end{equation}
where each $\widetilde{g}_{i,c}(\cdot; \vect{w}_{i,-c}[n])$ is any surrogate function satisfying conditions F1-F3 on the variable $\vect{w}_{i,c}$. It is easy to check that the surrogate $\widetilde{g}_i$ in (\ref{eq:parallel_surrogate_function}) satisfies F1-F3 on the variable $\vect{w}_i$. Given (\ref{eq:parallel_surrogate_function}), each core $c$ can then minimize its corresponding term independently of the others, and their solutions can be aggregated to form the final solution vector. In the case of the FL strategy, parallelization is not particularly effective. In fact, the final solution is given by simple aggregation of vectors as in \eqref{eq:ridge_FL}, whose computation has linear complexity with respect to the size of $\vect{w}_i$, eventually with a point-wise application of the thresholding operator in \eqref{eq:soft_thresh}. However, the (linear) cost of solving the surrogate problems at each core is easily overshadowed by the need of computing gradients via a backpropagation step.
On the other side, parallelization can largely reduce computational complexity when using the PL strategy. To give an example of application of the proposed methodology, in the sequel we illustrate how to parallelize the local optimization in the case of a ridge regression cost as in Sec. \ref{sec:ridge_regression}. In particular, let us consider the surrogate function in \eqref{eq:surrogate_ridge_pl}. To obtain the surrogate function associated to each core $c$, we fix in \eqref{eq:surrogate_ridge_pl} all the variables $\vect{w}_{i,-c}[n]$, such that the resulting function depends only on $\vect{w}_{i,c}$. The surrogate associated to core $c$ is then given by:
\begin{align}\label{eq:rige_parallel}
\widetilde{U}_{i,c}(\vect{w}_{i,c}; \vect{w}_{i,-c}[n]) & = \vect{w}_{i,c}^T \Bigl( \vect{A}_{i,c,c}[n] + \frac{\lambda}{2}\vect{I} \Bigr)\vect{w}_{i,c} \, - \nonumber \\ & 2\left(\vect{b}_{i,c}[n] - \vect{A}_{i,c,-c}[n]\vect{w}_{i,-c}[n] \right)^T\vect{w}_i \,,
\end{align}
where $\vect{A}_{i,c,c}[n]$ is the block (rows and columns) of the matrix $\vect{A}_i[n]$ in (\ref{eq:Ai}) corresponding to the $c$-th partition, whereas $\vect{A}_{i,c,-c}[n]$ takes the rows corresponding to the $c$-th partition and all the columns not associated to $c$. The minimum of (\ref{eq:rige_parallel}) is:
\begin{empheq}[box=\widefbox]{equation}\label{eq:rige_parallel_solution}
\widetilde{\vect{w}}_{i,c}[n] = \left( \vect{A}_{i,c,c}[n] + \frac{\lambda}{2}\vect{I} \right)^{-1}\left(\vect{b}_{i,c}[n] - \vect{A}_{i,c,-c}[n]\vect{w}_{i,-c}[n] \right) \,,
\end{empheq}
$c=1,\ldots,C$, and the overall solution is given by $\widetilde{\vect{w}}_{i}[n]=(\widetilde{\vect{w}}_{i,c}[n])_{c=1}^C$.
As we can see from (\ref{eq:rige_parallel_solution}), the effect of the parallelization is evident: At each iteration $n$, each core has to invert a matrix having (approximately) size $\frac{1}{C}$ of the original one in (\ref{eq:ridge_PL}), thus remarkably reducing the overall computational burden. Similar arguments can be used also to parallelize the formulations in \eqref{eq:surrogate_lasso_pl}, \eqref{eq:group_lasso_fl}, and \eqref{eq:surrogate_cross_entropy_pl}.

\section{Experimental Validation}
\label{sec:experimental_validation}

In this section, we assess the performance of the proposed method via numerical simulations. We begin by analyzing the test error of the solutions obtained by the algorithms for some representative regression and classification datasets in Sections \ref{sec:results_regression_datasets} and \ref{sec:results_classification_datasets}, respectively. Then, we consider the convergence behaviors of the proposed framework, comparing it to centralized and distributed counterparts, in Section \ref{sec:analysis_of_convergence}. In Section \ref{sec:exploiting_parallelization}, we describe the speed-up achieved thanks to the parallelization strategy outlined before. Finally, we consider large-scale inference in Section \ref{sec:large_scale_experiment}. Python code to repeat the experiments is available under open-source license on the web.\footnote{\url{https://bitbucket.org/ispamm/parallel-and-distributed-neural-networks}} The code is built upon the Theano \citep{bergstra2010theano} and Lasagne\footnote{\url{https://github.com/Lasagne/Lasagne}} libraries.

\subsection{Experimental setup}
\label{sec:experimental_setup}

In all experiments, the original dataset is normalized so that both inputs and outputs lie in the $\left[0, 1\right]$ range. Then, the dataset is partitioned as follows. First, $20\%$ of the dataset is kept separate to test the algorithms. The remaining $80\%$ is partitioned evenly among a randomly generated network of $10$ agents. For simplicity, we consider networks with time-invariant, symmetric connectivity, such that every pair of agents have a $20\%$ probability of being connected, with the only requirement that the overall network is connected. An example of such connectivity is shown in Fig. \ref{fig:network_example}.

\begin{figure}
\centering
\includegraphics[scale=0.7]{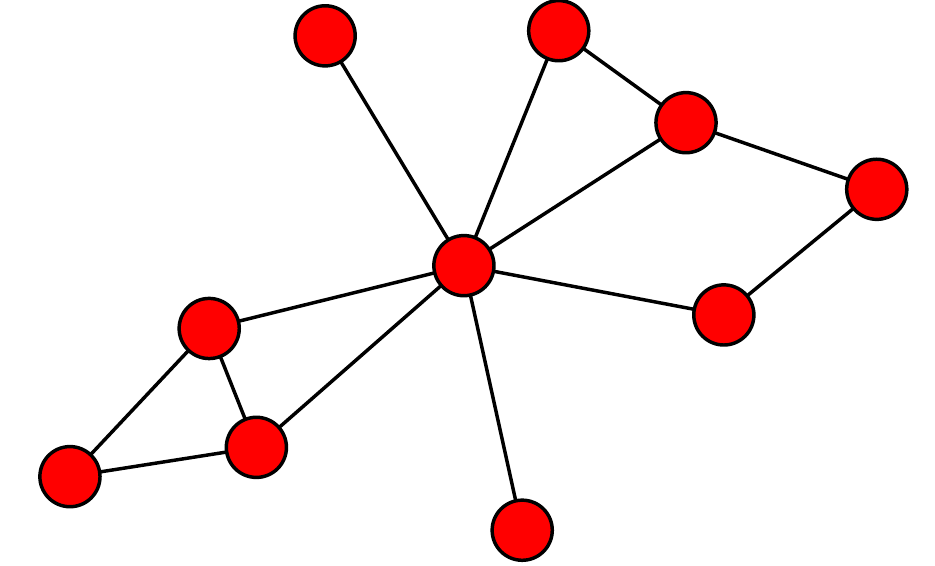}%
\caption{Example of communication network with $10$ agents (represented by red dots), possessing a sparse, time-invariant, symmetric connectivity.}
\label{fig:network_example}
\end{figure}

We have selected the weight coefficients in (\ref{weights}) using the \textit{Metropolis-Hastings} strategy \citep{lopes2008diffusion}:
\begin{equation}
C_{ij} =
	\begin{cases}
	\frac{1}{\max\{\delta_i, \delta_j\} + 1} & \; i \neq j,\, j \in \mathcal{N}_i \\
	1 - \sum_{j \in \mathcal{N}_i} \frac{1}{\max\{\delta_i, \delta_j\} + 1} & \; i = j \\
	0 & \; \text{otherwise}
	\end{cases}\,
\end{equation}
where $\delta_i$ is the degree of node $i$. It it easy to check that this choice of the weight matrix satisfies the convergence conditions of the framework. Missing data is handled by removing the corresponding example. All experiments are repeated $25$ times by varying the data partitioning and the NN initialization.

Regarding the NN structure, we use hyperbolic tangent nonlinearities in all neurons, except for classification problems, where we use a sigmoid nonlinearity in the output neuron. The weights of the NN are initialized independently at every agent using the normalized strategy described by \citet{glorot2010understanding}. All algorithms run for a maximum of $1000$ epochs. In all the figures illustrating the results of the distributed strategies, whenever not explicitly stated, we consider the evolution of the average weight vector $\overline{\vect{w}}[n]$ as defined in Proposition $2$.

\subsection{Results with regression datasets}
\label{sec:results_regression_datasets}

\begin{center}
\begin{table}
{\centering\hfill{}
\setlength{\tabcolsep}{4.5pt}
\renewcommand{\arraystretch}{1.3}
\begin{small}
\begin{tabular}{lccccc}   
\toprule
\textbf{Dataset} & \textbf{Features} & \textbf{Samples} & \textbf{NN Topology} & $\boldsymbol{\lambda}$ & \textbf{Source} \\
\midrule
Boston & $13$ & $506$ & $10$ & $10^{-1}$ & UCI \\
Kin8nm & $7$ & $8192$ & $8/5$ & $10^{-2}$ & DELVE \\
Wine & $10$ & $4898$ & $12$ & $10^{-2}$ & UCI \\
\bottomrule
\end{tabular}
\end{small}
}
\hfill{}
\caption{Schematic description of the datasets used for regression. For the NN topology, $x/y$ denotes a NN with two hidden layers of dimensions $x$ and $y$ respectively.}
\label{tab:datasets_regression}
\end{table}
\end{center}

\vspace{-2em}
We start considering three representative regression datasets, whose characteristics are summarized in Table \ref{tab:datasets_regression}. Boston (also known as the Housing dataset) is the task of predicting the median value of a house based on a set of features describing it \citep{quinlan1993combining}. Kin8nm is a member of the kinematics family of datasets\footnote{\url{http://www.cs.toronto.edu/~delve/data/kin/desc.html}}, having high non-linearity and a medium amount of additive noise. Finally, Wine concerns predicting the subjective quality of a (white) wine based on a wide set of chemical features \citep{cortez2009modeling}. The fourth and fifth columns in Table \ref{tab:datasets_regression} describe the parameters of the NN in terms of hidden neurons and regularization coefficients. These parameters are chosen based on an analysis of previous literature in order to obtain state-of-the-art results. However, we underline that our aim is to compare different solvers for the same NN optimization problem, and for this reason only relative differences in accuracy are of concern.

In particular, we compare the results of our algorithms with respect to five state-of-the-art \textit{centralized} solvers, in terms of mean-squared error (MSE) over the test data, when solving the global optimization problem with the ridge regression cost in \eqref{eq:ridge_regression}. Note that these solvers would not be available in a distributed scenario, and are only used for comparison purposes as optimal benchmarks. Specifically, we consider the following algorithms:
\begin{description}
	\item[\textbf{Gradient descent} (GD)]: this is a simple first-order steepest descent procedure with fixed step-size.
	\item[\textbf{AdaGrad}] \citep{duchi2011adaptive} : differently from GD, this algorithm employs different adaptive step-sizes per weight, which evolve according to the relative values of the gradients' updates.
	\item[RMSProp]: equivalent to an AdaDelta variant \citep{zeiler2012adadelta}, it also considers adaptive independent step-sizes; however they are adapted based on shorter time windows in order to avoid exponentially decreasing schedules.
	\item[Conjugate gradient (CG)]: this is the Polak-Ribiere variant of the nonlinear conjugate gradient algorithm \citep{nocedal2006numerical}, implemented in the SciPy library.\footnote{\url{https://scipy.org/}}
	\item[L-BFGS]: a low-memory version of the second-order BFGS algorithm \citep{byrd1995limited}, keeping track of an approximation to the full Hessian matrix, also implemented in the SciPy library.
\end{description}

In addition, we consider the behavior of a centralized implementation of the PL strategy in \eqref{eq:ridge_PL}, denoted as PL-SCA, resulting in a novel centralized algorithm. In particular, assuming all data is available on a single location, we can consider a centralized equivalent of \eqref{eq:Ai} and \eqref{eq:bi} as:
\begin{align}
\vect{A}[n] & = \sum_{i=1}^I \vect{A}_i[n] \,, \label{eq:A} \\
\vect{b}[n] & = \sum_{i=1}^I \sum_{m \in \mathcal{S}_i} \vect{J}_{i,m}[n]r_{i,m}[n] \,.
\end{align}
Following similar arguments as in Sections \ref{sec:next} and \ref{sec:ridge_regression}, PL-SCA is defined by the iterative application of the following recursion:
\begin{align}
\widetilde{\vect{w}}[n] & = \Bigl( \vect{A}[n] + \frac{\lambda}{2}\vect{I} \Bigr)^{-1}\vect{b}[n] \,, \\
\vect{w}[n+1] & = \vect{w}[n] + \alpha[n]\left( \widetilde{\vect{w}}[n] - \vect{w}[n] \right) \,.
\end{align}
The centralized implementation of the FL strategy is almost equivalent to GD, so we do not consider it separately. For the distributed algorithms, we consider both the PL strategy in \eqref{eq:ridge_PL}, denoted as PL-NEXT, and the FL strategy in \eqref{eq:ridge_FL}, denoted as FL-NEXT. For PL-SCA, PL-NEXT and FL-NEXT we use the quadratically decreasing step-size sequence defined in \eqref{eq:step_size_sequence_example}. To have a fair comparison, the parameters of the step-size sequence in (\ref{eq:step_size_sequence_example}) were tuned at hand in order to select the fastest convergence behavior for all algorithms.


\begin{sidewaystable}
{\centering\hfill{}
\setlength{\tabcolsep}{4pt}
\renewcommand{\arraystretch}{1.3}
\begin{footnotesize}
\begin{tabular}{lcccccccc}   
\toprule
\textbf{Dataset} & \multicolumn{6}{c}{\textbf{Centralized}} & \multicolumn{2}{c}{\textbf{Distributed}} \\ \cmidrule(lr){2-7} \cmidrule(lr){8-9}
								  & GD & AdaGrad & RMSProp & CG & L-BFGS & PL-SCA & FL-NEXT & PL-NEXT \\
\midrule
Boston & $0.010 \pm 0.001$ & $0.009 \pm 0.001$ & $0.009 \pm 0.001$ & $\vect{0.007 \pm 0.001}$ & $\vect{0.007 \pm 0.001}$ & $\vect{0.007 \pm 0.001}$ & $0.010 \pm 0.001$ & $\vect{0.007 \pm 0.001}$ \\
Kin8nm & $0.019 \pm \approx 0$ & $0.018 \pm \approx 0$ & $0.015 \pm 0.001$ & $0.011 \pm 0.001$ & $\vect{0.009 \pm 0.001}$ & $\vect{0.009 \pm \approx 0}$ & $0.019 \pm \approx 0$ & $\vect{0.009 \pm \approx 0}$ \\
Wine & $0.018 \pm 0.001$ & $0.016 \pm 0.001$ & $0.017 \pm 0.001$ & $0.038 \pm 0.019$ & $\vect{0.014 \pm \approx 0}$ & $\vect{0.014 \pm 0.001}$ & $0.017 \pm 0.001$ & $\vect{0.014 \pm 0.001}$ \\
\bottomrule
\end{tabular}
\end{footnotesize}
}
\hfill{}
\caption{Regression results (in terms of mean-squared error on the test set) of different algorithms. Columns $2$-$7$ are centralized algorithms reported for comparison. Columns $8$-$9$ are the two implementations of the proposed distributed framework. Best results for both groups are highlighted in bold. See the text for a description of the acronyms.}
\label{tab:results_regression_problems}
\end{sidewaystable}

The results on this set of experiments are provided in Table \ref{tab:results_regression_problems}, both in terms of the mean and the standard deviation. Several conclusions can be drawn from the table. For the centralized algorithms, L-BFGS, being a second-order algorithm, is able to obtain the best accuracies, and it is matched only by CG in the Boston case (in the next section we will show some plots of the convergence behavior of the different algorithms). Interestingly, PL-SCA is able to match L-BFGS in all cases, complementing our previous observation that the matrix in \eqref{eq:A} acts as an approximation of the Hessian matrix. For the distributed algorithms, we see similar distinctions between FL-NEXT and PL-NEXT. Specifically, PL-NEXT has comparable accuracies with respect to L-BFGS, while FL-NEXT obtains errors comparable to GD and AdaGrad. Clearly, the improved convergence comes at the cost of a higher computational burden (due to the need of inverting a matrix in \eqref{eq:ridge_PL}), in line with the equivalent difference in the centralized case. Summarizing, we see that FL-NEXT and PL-NEXT represent viable algorithms for distributed scenarios, providing a relative trade-off with respect to convergence and computational requirements, and matching the respective centralized implementations that are not viable in the distributed setting treated in this paper. Importantly, this is also achieved with a minimal (or non-existent) increase in term of variance. We defer a statistical analysis of the results to the next section, in order to consider also the classification datasets.

\subsection{Results with classification datasets}
\label{sec:results_classification_datasets}

In this section, we analyze the performance of the distributed algorithms when applied to two classification problems, whose characteristics are briefly summarized in Table \ref{tab:datasets_classification}. Wisconsin is a medical classification task, aimed at separating cancerous cells from non-cancerous ones from several features describing the cell nucleus.\footnote{\url{https://archive.ics.uci.edu/ml/datasets/Breast+Cancer+Wisconsin+(Diagnostic)}} The Cardiotocography (CGT) dataset is another clinical problem, where we wish to infer suspect/pathological fetuses from several biometric signals.\footnote{\url{https://archive.ics.uci.edu/ml/datasets/Cardiotocography}} In this case, we solve the global optimization problem with the cross-entropy loss in \eqref{eq:cross_entropy_loss} and a squared regularization term. We analyze the behavior of both the FL strategy and the PL strategy when compared to the state-of-the-art solvers described in the previous section. For PL-NEXT, the local surrogate problem in \eqref{eq:surrogate_cross_entropy_pl} is solved with AdaGrad, run for a maximum of $50$ iterations (with an initial step-size of $0.1$), or until the gradient norm is below a fixed threshold of $10^{-6}$. For the local optimization at each agent, we perform a `warm start' from the current estimate $\vect{w}_i[n]$.

\begin{center}
\begin{table}
{\centering\hfill{}
\setlength{\tabcolsep}{4pt}
\renewcommand{\arraystretch}{1.3}
\begin{small}
\begin{tabular}{lccccc}   
\toprule
\textbf{Dataset} & \textbf{Features} & \textbf{Samples} & \textbf{NN Topology} & $\boldsymbol{\lambda}$ & \textbf{Source} \\
\midrule
Wisconsin & $9$ & $689$ & $10$ & $10^{-0.5}$ & UCI \\
CTG & $28$ & $2126$ & $15/8$ & $10$ & UCI \\
\bottomrule
\end{tabular}
\end{small}
}
\hfill{}
\caption{Schematic description of the datasets used for classification. See Table \ref{tab:datasets_regression} and the text for details on the NN topology.}
\label{tab:datasets_classification}
\end{table}
\end{center}

\begin{sidewaystable*}
{\centering\hfill{}
\setlength{\tabcolsep}{4pt}
\renewcommand{\arraystretch}{1.3}
\begin{small}
\begin{tabular}{lccccccc}   
\toprule
\textbf{Dataset} & \multicolumn{5}{c}{\textbf{Centralized}} & \multicolumn{2}{c}{\textbf{Distributed}} \\ \cmidrule(lr){2-6} \cmidrule(lr){7-8}
								  & GD & AdaGrad & RMSprop & CG & L-BFGS & FL-NEXT & PL-NEXT \\
\midrule
Wisconsin & $\vect{0.025 \pm 0.009}$ & $0.027 \pm 0.006$ & $0.027 \pm 0.005$ & $0.028 \pm 0.006$ & $0.028 \pm 0.006$ & $0.027 \pm 0.007$ & $\vect{0.025 \pm 0.009}$ \\
CTG & $0.084 \pm 0.010$ & $0.082 \pm 0.010$ & $0.087 \pm 0.014$ & $\vect{0.083 \pm 0.011}$ & $\vect{0.083 \pm 0.010}$ & $0.087 \pm 0.007$ & $\vect{0.084 \pm 0.009}$ \\
\bottomrule
\end{tabular}
\end{small}
}
\hfill{}
\caption{Classification results (in terms of misclassification rate on the test set) of different algorithms. Columns $2$-$6$ are centralized algorithms reported for comparison. Columns $7$-$8$ are the two implementations of the proposed distributed framework. Best results for both groups are highlighted in bold. See the text for a description of the acronyms.}
\label{tab:results_classification_problems}
\end{sidewaystable*}

\vspace{-2.5em}
The overall results are given in Table \ref{tab:results_classification_problems} in terms of misclassification rate. We see that, in this case, first-order algorithms are generally competitive, with the GD solver obtaining the best accuracy among the centralized solvers for the Wisconsin dataset, and CG/L-BFGS obtaining a slightly better result in the CTG case. Nevertheless, the distributed strategies are again able to obtain state-of-the-art results, with PL-NEXT consistently obtaining the lowest misclassification rate, and FL-NEXT ranging close to AdaGrad and RMSProp. In order to formalize the intuition that PL-NEXT is generally converging to a better minimum than FL-NEXT, we perform a Wilcoxon signed-rank test \citep{demsar2006statistical} on the results over both regression and classification datasets. The difference is found to be significant with a $p=0.05$ confidence value (although the number of datasets under consideration is relatively small). We can reasonably conclude that PL-NEXT seems a better choice in terms of accuracy, if it is possible for the agents to cope with the increased computational cost.

\subsection{Analysis of convergence}
\label{sec:analysis_of_convergence}

In a distributed setting, the final accuracy is not the only parameter of interest. We are also concerned on how fast this accuracy is obtained, because the convergence speed has a direct impact on the communication burden over the network of agents. As we mentioned in the introduction, in the case of general non-differentiable regularizers $r$, there is no ready-to-use alternative for comparing our proposed algorithms. However, in the specific case where the regularization function $r$ satisfies assumption B1, we can easily adapt the framework introduced in \citet{bianchi2013convergence}, resulting in a simple method that we denote as the distributed gradient (DistGrad) algorithm. Similarly to the NEXT framework, DistGrad alternates between a local optimization phase and a communication phase. In the optimization phase, each agent iteratively updates its own estimate according to a local gradient descent step as follows:
\begin{equation}
\vect{z}_{i}[n] = \vect{w}_i[n] - \eta[n]\left( \nabla \vect{g}_i[n] + \frac{1}{I} \nabla r(\vect{w}_i[n]) \right) \,,
\label{eq:distgrad}
\end{equation}
where $\eta[n]$ is the step-size sequence. In the communication phase, the local estimates $\vect{z}_{i}[n]$ are combined similarly to \eqref{consensus_update}. DistGrad can be seen as a simplified version of the FL strategy, where we do not consider the dynamic consensus step (i.e., Step 3 of NEXT). For fairness of comparison, we use the step-size rule in \eqref{eq:step_size_sequence_example}, and the same strategy for selecting the combination coefficients in (\ref{weights}).

\begin{figure*}
	\centering
	\subfloat[Cost function (Boston)]{\includegraphics[scale=0.8]{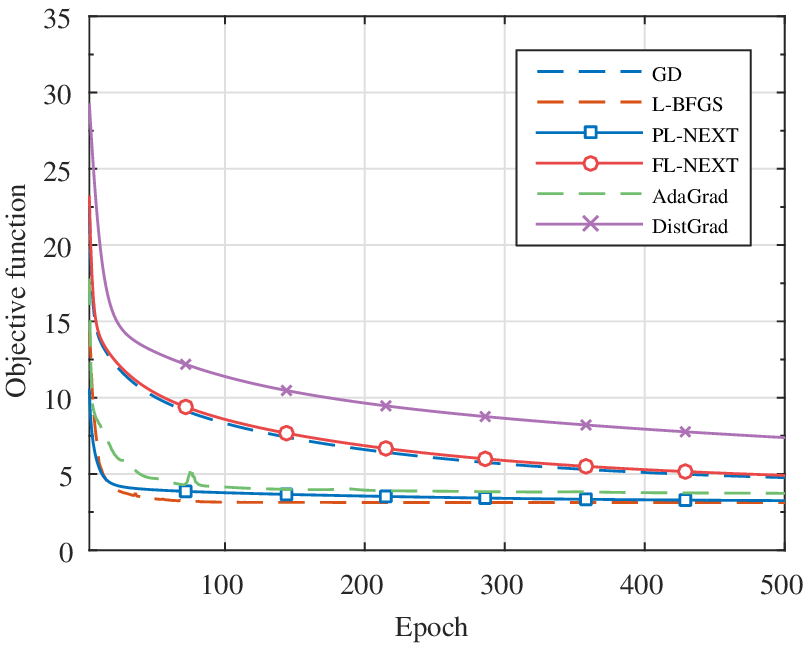}%
		\label{fig:boston_obj}}
	\hfil
	\subfloat[Cost function (Kin8nm)]{\includegraphics[scale=0.8]{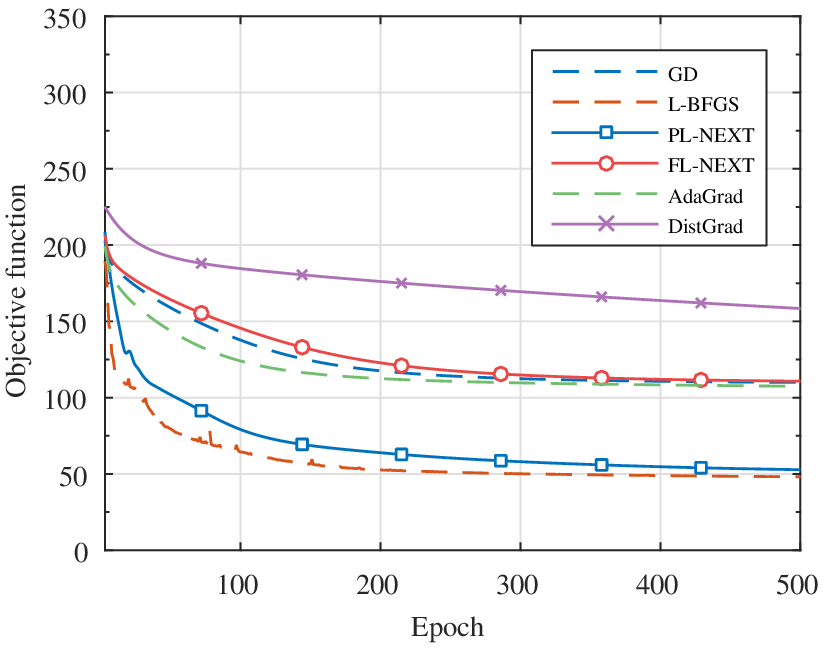}%
		\label{fig:kin8nm_obj}}
	\vfill
	\subfloat[Test error (Boston)]{\includegraphics[scale=0.8]{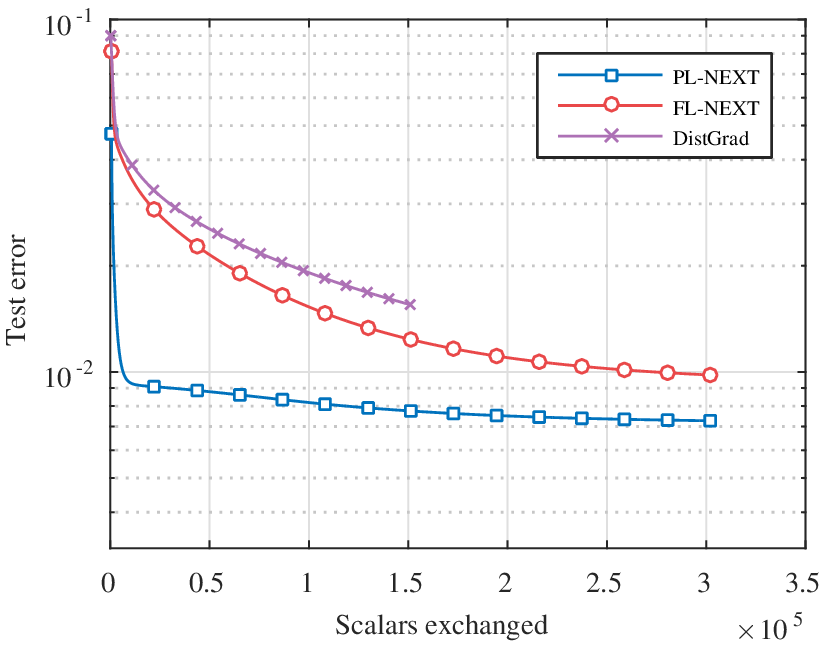}%
			\label{fig:boston_test_error}}
	\hfil
	\subfloat[Test error (Kin8nm)]{\includegraphics[scale=0.8]{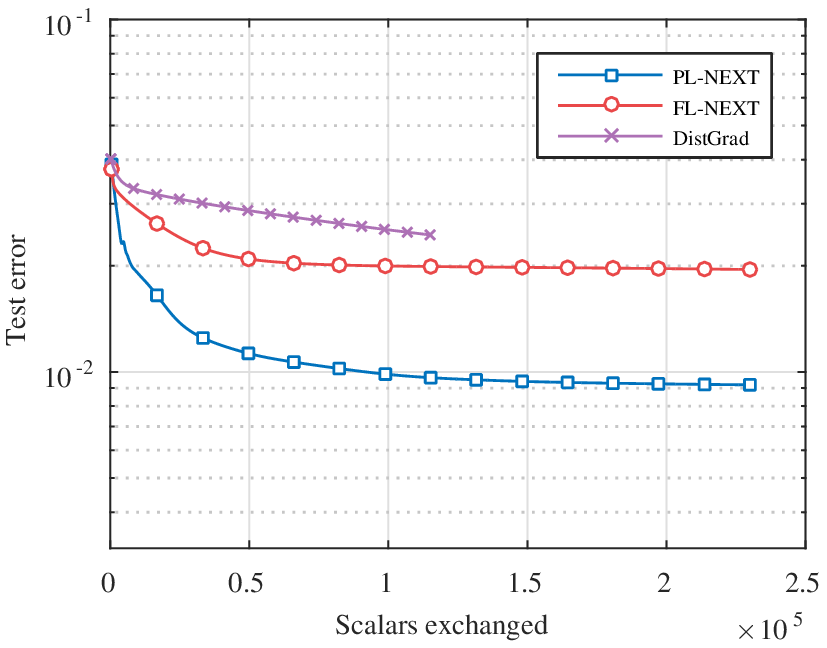}%
		\label{fig:kin8nm_test_error}}
	\vfill
	\subfloat[Disagreement (Boston)]{\includegraphics[scale=0.8]{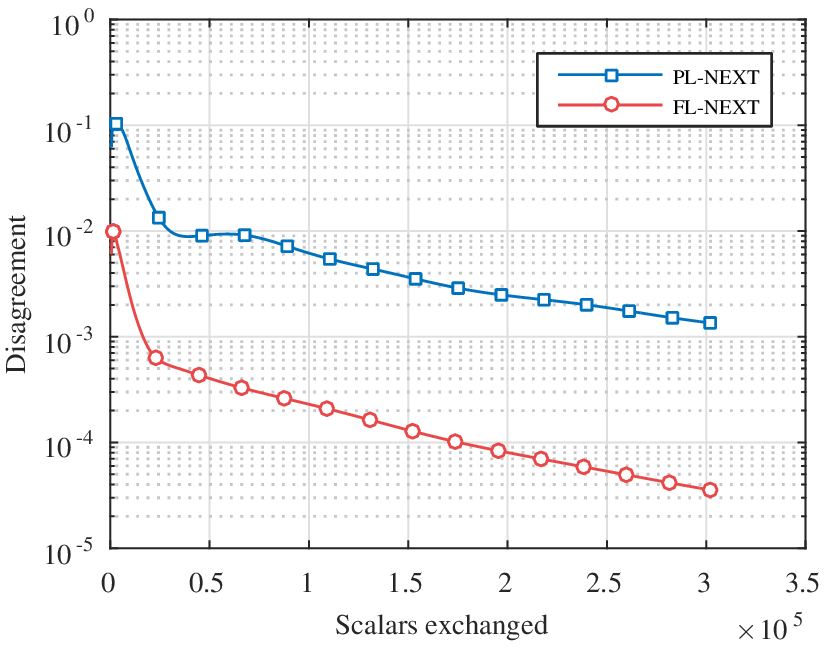}%
		\label{fig:boston_dis}}
	\hfil
	\subfloat[Disagreement (Kin8nm)]{\includegraphics[scale=0.8]{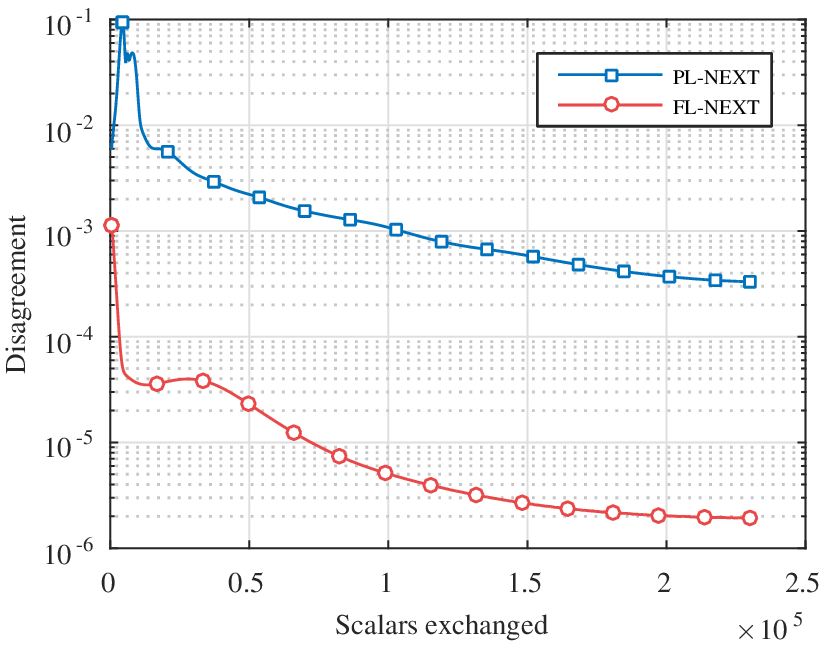}%
		\label{fig:kin8nm_dis}}
	\caption{(a-b) Cost function value (per epoch). (c-d) Test error (per scalars exchanged). (e-f) Evolution of the disagreement (per scalars exchanged). Graphs on the left column are for the Boston dataset, graphs on the right column are for the Kin8nm dataset. For readability, centralized algorithms are represented with dashed lines, while distributed algorithms are represented with solid lines with specific markers.}
	\label{fig:boston_and_kin8nm}
\end{figure*}

In Figs. \ref{fig:boston_obj}-\ref{fig:kin8nm_obj} we plot the evolution of the global cost function in \eqref{Dist_NN_training} for FL-NEXT, PL-NEXT, DistGrad and a few representative centralized solvers for two different datasets. For improved readability, the behavior of centralized solvers is depicted using dashed lines, while the distributed algorithms are shown with solid lines. We see that the results are similar to what we have already discussed previously for the final test error: PL-NEXT is able to track consistently the convergence rate of L-BFGS, while FL-NEXT achieves results comparable to (centralized) first-order procedures. Differently, the DistGrad algorithm is slower and, for a given number of epochs, has a very large gap compared to other methods.

Another performance metric of interest is the transient behavior of the test error in terms of the amount of scalar values that are exchanged among agents in the network. We plot this metric for the three distributed algorithms in Figs. \ref{fig:boston_test_error}-\ref{fig:kin8nm_test_error}, where the $y$-axis is shown with a logarithmic scale. We notice that DistGrad requires exactly half as many scalars to be exchanged at every iteration (since it does not rely on the dynamic consensus to track the average gradient). Nevertheless, from Figs. \ref{fig:boston_test_error}-\ref{fig:kin8nm_test_error}, we can see that both PL-NEXT and FL-NEXT can reach better errors with respect to DistGrad for any given amount of scalars exchanged, showing their better efficiency in terms of overall communication burden. PL-NEXT is particularly well performing, with only a very small amount of communication required for achieving an error close to the optimal one.

A final metric of interest is the average disagreement among the agents, which is computed as:
\begin{equation}\label{disagreement}
D[n] \triangleq \frac{1}{I} \norm[\infty]{ \vect{w}_i[n] - \overline{\vect{w}}[n] } \,.
\end{equation}
We plot the behavior of (\ref{disagreement}) for PL-NEXT and FL-NEXT in Figs. \ref{fig:boston_dis}-\ref{fig:kin8nm_dis}, where we can see that both algorithms rapidly tend to reach a consensus among the different agents in the network.

\subsection{Exploiting parallelization}
\label{sec:exploiting_parallelization}

\begin{figure*}[!t]
	\centering
	\subfloat[Speedup (Boston)]{\includegraphics[scale=0.8]{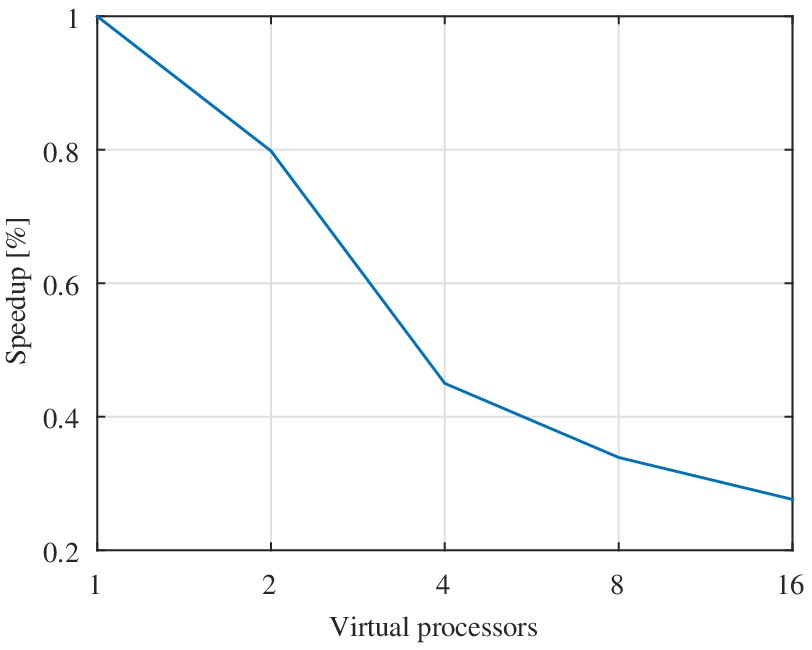}%
		\label{fig:boston_parallel_time}}
	\hfil
	\subfloat[Speedup (Kin8nm)]{\includegraphics[scale=0.8]{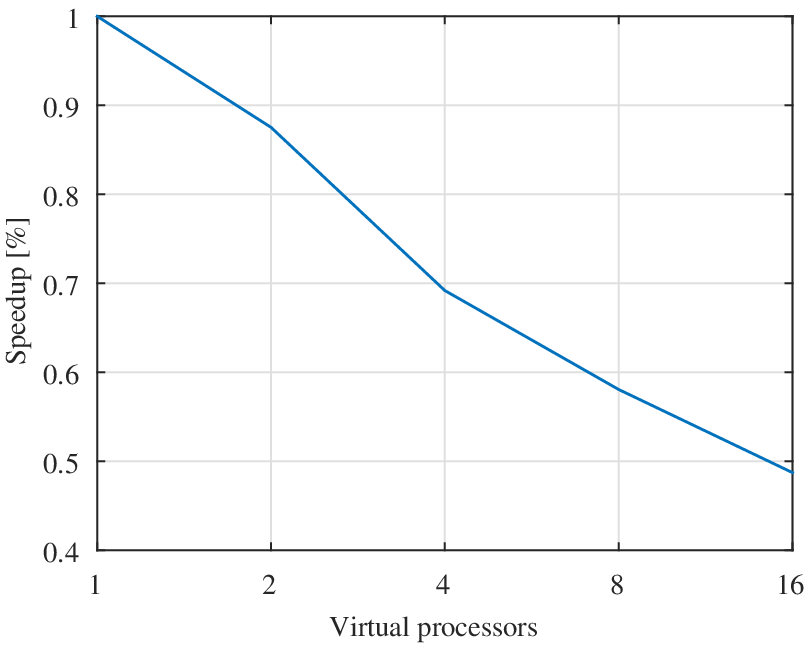}%
		\label{fig:kin8nm_parallel_time}}
	\vfill
	\subfloat[Cost function (Boston)]{\includegraphics[scale=0.8]{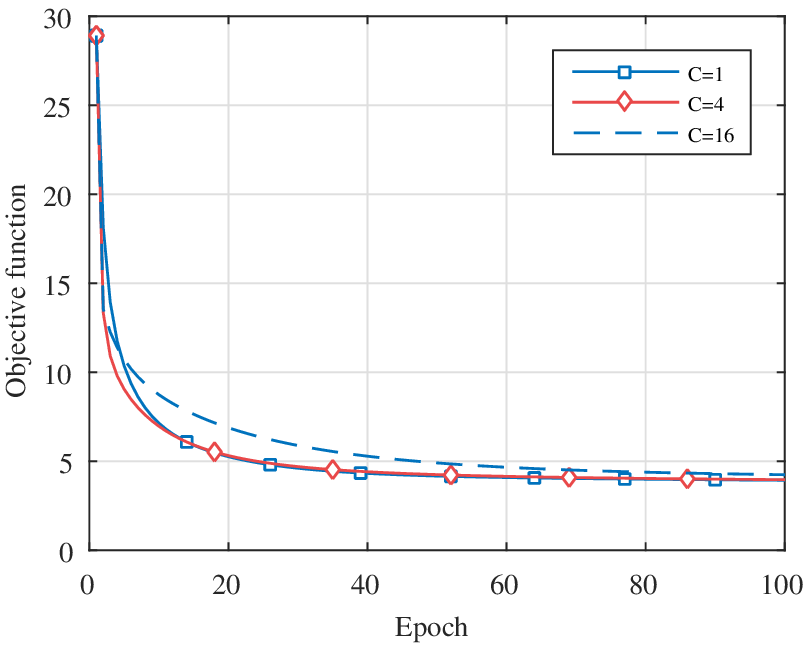}%
		\label{fig:boston_parallel_obj}}
	\hfil
	\subfloat[Cost function (Kin8nm)]{\includegraphics[scale=0.8]{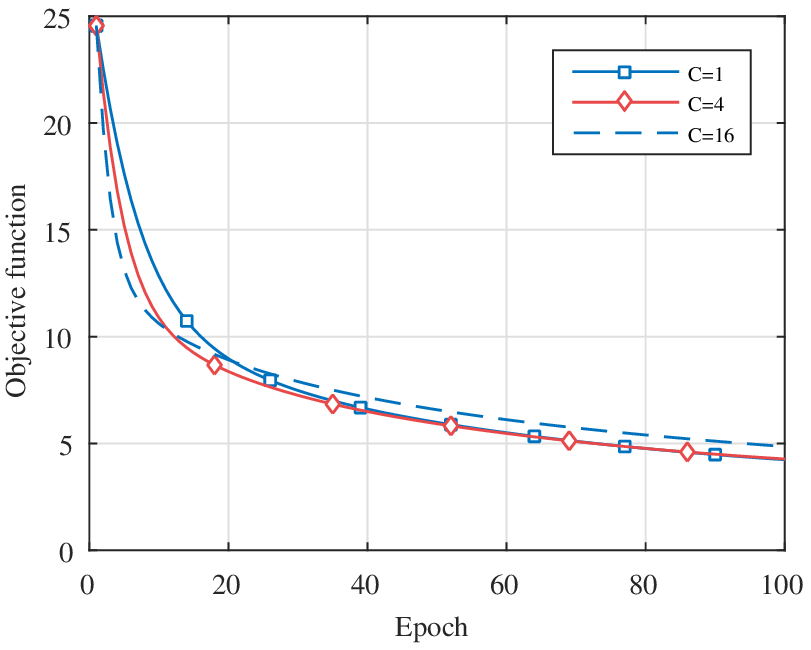}%
		\label{fig:kin8nm_parallel_obj}}
	\caption{(a-b) Relative speedup (per number of local processors). (c-d) Cost function value (per epoch). Graphs on the left column are for the Boston dataset, graphs on the right column are for the Kin8nm dataset.}
	\label{fig:boston_and_kin8nm_parallel}
\end{figure*}

Next, we investigate the speed-up obtained by parallelizing the local optimization at each agent. We consider again the Boston and Kin8nm datasets, but we vary the number of (local) processors available at every agent in the range $2^j$, with $j=0, 1, \ldots, 4$. The relative speedup with respect to the baseline $C=1$ is shown in Figs. \ref{fig:boston_parallel_time}-\ref{fig:kin8nm_parallel_time}. We see that the speedup is roughly linear with respect to the amount of available processors, so that in the case $C=16$ we only need $\approx \frac{1}{3}$ of the time for Boston, and $\approx \frac{1}{2}$ for Kin8nm. Additionally, in Figs. \ref{fig:boston_parallel_obj}-\ref{fig:kin8nm_parallel_obj}, we can visualize the evolution of the overall cost function for $C=1$, $C=4$ and $C=16$. From the figures, we can notice that the improvement in training time is obtained with only a limited effect on the convergence behavior, where in the worst case we obtain only a very small (or null) decrease.

\subsection{Experiment on a large-scale dataset}
\label{sec:large_scale_experiment}

Before concluding this experimental section, we briefly discuss the important point of large-scale distributed learning, i.e., performing distributed inference whenever $N_k$ is very large for the majority of the agents. To this end, we consider the YearPredictionMSD dataset \citep{bertin2011million}, which is one of the largest regression datasets available on the UCI repository. The task is to predict the year of release of a song starting from $90$ audio features. There are $463,715$ examples for training, and $51,630$ examples for testing (of different authors). Similarly to before, we preprocess the input and output values in $\left[0, 1\right]$, and we consider a NN with two hidden layers having, respectively, $40$ and $20$ neurons. We partition the training data among $10$ different agents, and we compare PL-NEXT with AdaGrad. We choose AdaGrad for two main reasons, i.e., it was found to be extremely fast in the previous section, and we can use it together with stochastic updates over small batches of data in order to handle the large-scale dataset. Specifically, for every iteration AdaGrad is updated with mini-batches of $500$ elements, and accuracy is computed after a complete pass over the training dataset. The regularization is chosen as $\lambda = 10$. Step-sizes are chosen in order to guarantee a smooth convergence behavior.

\begin{figure}
\centering
\includegraphics[scale=0.8]{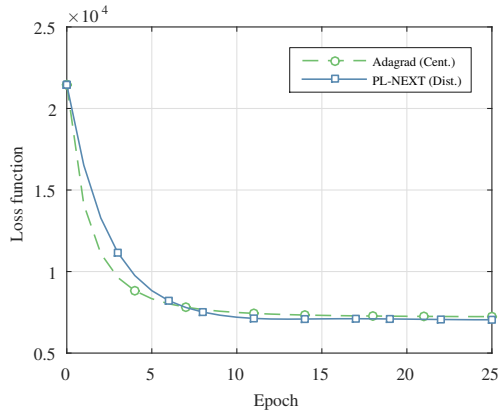}
\caption{Average evolution of the loss on the MSD dataset (see the text for a full description). For AdaGrad, one epoch corresponds to an entire pass over the training data.}
\label{fig:msd_loss_evolution}
\end{figure}

The evolution of the global loss function in \eqref{Dist_NN_training} is shown in Fig. \ref{fig:msd_loss_evolution}. Despite AdaGrad making several stochastic update steps at every iteration, PL-NEXT is able to achieve a comparable convergence behavior, with a minimum loss value which is slightly better due to the unnoisy gradient evaluations. Both algorithms also achieve a similar mean squared error on the independent test set, which is around $0.011$. For comparison, the average MSE of a support vector algorithm is around $0.013/0.014$ \citep{ho2012large}.

This example shows two important aspects of large-scale inference over networks. First of all, what is considered a challenging benchmark in a centralized environment might be relatively simpler in a distributed experiment, since the training data must be partitioned over several agents. In this case, for example, the original half a million training points must be partitioned over $10$ agents, so that each agent only has to deal with $\approx 50,000$ training points. Thus, there is the need of designing more elaborate benchmarks to test the capabilities of the algorithms is larger situations. Secondly, properly handling these datasets will require the development of stochastic updates at every agent, paralleling the stochastic algorithms used in the centralized case and commonly used in the deep learning literature. Having such stochastic algorithms for distributed, non-convex problems remain an open problem in the literature, and we remark it here as the main line of research for future investigations.

\section{Conclusions and Future Works}
\label{sec:conclusion}

In this paper, we have investigated the problem of training a NN model in a distributed scenario, where multiple agents have a limited knowledge of the training dataset. We have proposed a provably convergent procedure to this aim, which builds exclusively on local optimization steps and one-hop communication steps. The method can be customized to several typical error functions and regularization terms. We have also described an immediate way to parallelize the local optimization phase across multiple processors/machines, available at each agent, with a limited impact on the convergence behavior.

One immediate extension of the framework presented here is to handle non-convex regularization terms, which are generally considered too challenging in practice. One example is the sample variance penalization \citep{maurer2009empirical}, which is defined in terms of the NN output. Additional extensions can consider the presence of non-differentiable points in the NN model (e.g., by using ReLu activation functions), stochastic updates of the surrogate functions, and online formulations where new data arrives in a streaming fashion, like in distributed filtering \citep{sayed2014adaptive}. Some interesting results can derive by considering the literature on distributed constraint optimization problems (DCOP), which deals with distributed decision making problems where the decision variables are separated among the different agents \citep{modi2005adopt,rogers2011bounded}. Finally, we are interested in testing our framework on real-world applications such as, e.g., multimedia classification and chaotic prediction tasks.

\bibliographystyle{elsarticle-harv}
\bibliography{biblio}

\begin{thebibliography}{64}
\expandafter\ifx\csname natexlab\endcsname\relax\def\natexlab#1{#1}\fi
\expandafter\ifx\csname url\endcsname\relax
  \def\url#1{\texttt{#1}}\fi
\expandafter\ifx\csname urlprefix\endcsname\relax\def\urlprefix{URL }\fi

\bibitem[{Abadi et~al.(2016)Abadi, Agarwal, Barham, Brevdo, Chen, Citro,
  Corrado, Davis, Dean, Devin, et~al.}]{abadi2016tensorflow}
Abadi, M., Agarwal, A., Barham, P., Brevdo, E., Chen, Z., Citro, C., Corrado,
  G.~S., Davis, A., Dean, J., Devin, M., et~al., 2016. Tensorflow: Large-scale
  machine learning on heterogeneous distributed systems. arXiv preprint
  arXiv:1603.04467.

\bibitem[{Beck and Teboulle(2009)}]{beck2009fast}
Beck, A., Teboulle, M., 2009. A fast iterative shrinkage-thresholding algorithm
  for linear inverse problems. SIAM Journal on Imaging Sciences 2~(1),
  183--202.

\bibitem[{Bengio(2012)}]{bengio2012practical}
Bengio, Y., 2012. Practical recommendations for gradient-based training of deep
  architectures. In: Neural Networks: Tricks of the Trade. Springer, pp.
  437--478.

\bibitem[{Bergstra et~al.(2010)Bergstra, Breuleux, Bastien, Lamblin, Pascanu,
  Desjardins, Turian, Warde-Farley, and Bengio}]{bergstra2010theano}
Bergstra, J., Breuleux, O., Bastien, F., Lamblin, P., Pascanu, R., Desjardins,
  G., Turian, J., Warde-Farley, D., Bengio, Y., 2010. Theano: A cpu and gpu
  math compiler in python. In: Proceedings of the 9th Python in Science
  Conference. pp. 1--7.

\bibitem[{Bertin-Mahieux et~al.(2011)Bertin-Mahieux, Ellis, Whitman, and
  Lamere}]{bertin2011million}
Bertin-Mahieux, T., Ellis, D.~P., Whitman, B., Lamere, P., 2011. The million
  song dataset. In: 12th International Society for Music Information Retrieval
  Conference (ISMIR). pp. 1--6.

\bibitem[{Bianchi and Jakubowicz(2013)}]{bianchi2013convergence}
Bianchi, P., Jakubowicz, J., 2013. Convergence of a multi-agent projected
  stochastic gradient algorithm for non-convex optimization. IEEE Transactions
  on Automatic Control 58~(2), 391--405.

\bibitem[{Bishop(2006)}]{bishop2006pattern}
Bishop, C.~M., 2006. Pattern recognition and machine learning. Springer
  International.

\bibitem[{Blackwell(2012)}]{blackwell2012neural}
Blackwell, W.~J., 2012. Neural network jacobian analysis for high-resolution
  profiling of the atmosphere. EURASIP Journal on Advances in Signal Processing
  2012~(1), 1.

\bibitem[{Boric-Lubeke and Lubecke(2002)}]{boric2002wireless}
Boric-Lubeke, O., Lubecke, V.~M., 2002. Wireless house calls: using
  communications technology for health care and monitoring. IEEE Microwave
  Magazine 3~(3), 43--48.

\bibitem[{Boyd and Vandenberghe(2004)}]{boyd2004convex}
Boyd, S., Vandenberghe, L., 2004. Convex optimization. Cambridge university
  press.

\bibitem[{Byrd et~al.(1995)Byrd, Lu, Nocedal, and Zhu}]{byrd1995limited}
Byrd, R.~H., Lu, P., Nocedal, J., Zhu, C., 1995. A limited memory algorithm for
  bound constrained optimization. SIAM Journal on Scientific Computing 16~(5),
  1190--1208.

\bibitem[{Cevher et~al.(2014)Cevher, Becker, and Schmidt}]{cevher2014convex}
Cevher, V., Becker, S., Schmidt, M., 2014. Convex optimization for big data:
  Scalable, randomized, and parallel algorithms for big data analytics. IEEE
  Signal Processing Magazine 31~(5), 32--43.

\bibitem[{Cortez et~al.(2009)Cortez, Cerdeira, Almeida, Matos, and
  Reis}]{cortez2009modeling}
Cortez, P., Cerdeira, A., Almeida, F., Matos, T., Reis, J., 2009. Modeling wine
  preferences by data mining from physicochemical properties. Decision Support
  Systems 47~(4), 547--553.

\bibitem[{Dean et~al.(2012)Dean, Corrado, Monga, Chen, Devin, Mao, Senior,
  Tucker, Yang, Le, et~al.}]{dean2012large}
Dean, J., Corrado, G., Monga, R., Chen, K., Devin, M., Mao, M., Senior, A.,
  Tucker, P., Yang, K., Le, Q.~V., et~al., 2012. Large scale distributed deep
  networks. In: Advances in neural information processing systems. pp.
  1223--1231.

\bibitem[{Dem{\v{s}}ar(2006)}]{demsar2006statistical}
Dem{\v{s}}ar, J., 2006. {Statistical Comparisons of Classifiers over Multiple
  Data Sets}. Journal of Machine Learning Research 7, 1--30.

\bibitem[{Di~Lorenzo and Sayed(2013)}]{di2013sparse}
Di~Lorenzo, P., Sayed, A.~H., 2013. Sparse distributed learning based on
  diffusion adaptation. IEEE Transactions on Signal Processing 61~(6),
  1419--1433.

\bibitem[{Di~Lorenzo and Scardapane(2016)}]{di2016neuralnetworks}
Di~Lorenzo, P., Scardapane, S., 2016. Parallel and distributed training of
  neural networks via successive convex approximation. In: 2016 IEEE
  International Workshop on Machine Learning for Signal Processing (MLSP).
  IEEE, pp. 1--6.

\bibitem[{Di~Lorenzo and Scutari(2016)}]{di2016next}
Di~Lorenzo, P., Scutari, G., 2016. Next: In-network nonconvex optimization.
  IEEE Transactions on Signal and Information Processing over Networks 2~(2),
  120--136.

\bibitem[{Duchi et~al.(2011)Duchi, Hazan, and Singer}]{duchi2011adaptive}
Duchi, J., Hazan, E., Singer, Y., 2011. Adaptive subgradient methods for online
  learning and stochastic optimization. Journal of Machine Learning Research
  12~(Jul), 2121--2159.

\bibitem[{Facchinei et~al.(2015)Facchinei, Scutari, and
  Sagratella}]{facchinei2015parallel}
Facchinei, F., Scutari, G., Sagratella, S., 2015. Parallel selective algorithms
  for nonconvex big data optimization. IEEE Transactions on Signal Processing
  63~(7), 1874--1889.

\bibitem[{Forero et~al.(2010)Forero, Cano, and Giannakis}]{forero2010consensus}
Forero, P.~A., Cano, A., Giannakis, G.~B., 2010. Consensus-based distributed
  support vector machines. Journal of Machine Learning Research 11~(May),
  1663--1707.

\bibitem[{Gao et~al.(2015)Gao, Chen, Richard, and Huang}]{gao2015diffusion}
Gao, W., Chen, J., Richard, C., Huang, J., 2015. Diffusion adaptation over
  networks with kernel least-mean-square. In: Computational Advances in
  Multi-Sensor Adaptive Processing (CAMSAP), 2015 IEEE 6th International
  Workshop on. IEEE, pp. 217--220.

\bibitem[{Georgopoulos and Hasler(2014)}]{georgopoulos2014distributed}
Georgopoulos, L., Hasler, M., 2014. Distributed machine learning in networks by
  consensus. Neurocomputing 124, 2--12.

\bibitem[{Glorot and Bengio(2010)}]{glorot2010understanding}
Glorot, X., Bengio, Y., 2010. Understanding the difficulty of training deep
  feedforward neural networks. In: AISTATS. Vol.~9. pp. 249--256.

\bibitem[{Glorot et~al.(2011)Glorot, Bordes, and Bengio}]{glorot2011deep}
Glorot, X., Bordes, A., Bengio, Y., 2011. {Deep Sparse Rectifier Neural
  Networks}. In: Proc. 14th International Conference on Artificial Intelligence
  and Statistics. pp. 315--323.

\bibitem[{Goodfellow et~al.(2013)Goodfellow, Warde-farley, Mirza, Courville,
  and Bengio}]{goodfellow2013maxout}
Goodfellow, I.~J., Warde-farley, D., Mirza, M., Courville, A., Bengio, Y.,
  2013. {Maxout networks}. In: Proc. 30th International Conference on Machine
  Learning. pp. 1319--1327.

\bibitem[{Haykin(2009)}]{haykin2009neural}
Haykin, S., 2009. Neural networks and learning machines, 3rd Edition. Pearson.

\bibitem[{Ho and Lin(2012)}]{ho2012large}
Ho, C.-H., Lin, C.-J., 2012. Large-scale linear support vector regression.
  Journal of Machine Learning Research 13~(Nov), 3323--3348.

\bibitem[{Huang and Li(2015)}]{huang2015distributed}
Huang, S., Li, C., 2015. Distributed extreme learning machine for nonlinear
  learning over network. Entropy 17~(2), 818--840.

\bibitem[{Lazarevic and Obradovic(2002)}]{lazarevic2002boosting}
Lazarevic, A., Obradovic, Z., 2002. Boosting algorithms for parallel and
  distributed learning. Distributed and Parallel Databases 11~(2), 203--229.

\bibitem[{LeCun et~al.(2015)LeCun, Bengio, and Hinton}]{lecun2015deep}
LeCun, Y., Bengio, Y., Hinton, G., 2015. Deep learning. Nature 521~(7553),
  436--444.

\bibitem[{Lopes and Sayed(2008)}]{lopes2008diffusion}
Lopes, C.~G., Sayed, A.~H., 2008. {D}iffusion least-mean squares over adaptive
  networks: {F}ormulation and performance analysis. {IEEE} {T}ransactions on
  {S}ignal {P}rocessing 56~(7), 3122--3136.

\bibitem[{Lu et~al.(2008)Lu, Roychowdhury, and
  Vandenberghe}]{lu2008distributed}
Lu, Y., Roychowdhury, V., Vandenberghe, L., 2008. Distributed parallel support
  vector machines in strongly connected networks. IEEE Transactions on Neural
  Networks 19~(7), 1167--1178.

\bibitem[{Mateos et~al.(2010)Mateos, Bazerque, and
  Giannakis}]{mateos2010distributed}
Mateos, G., Bazerque, J.~A., Giannakis, G.~B., 2010. Distributed sparse linear
  regression. IEEE Transactions on Signal Processing 58~(10), 5262--5276.

\bibitem[{Maurer and Pontil(2009)}]{maurer2009empirical}
Maurer, A., Pontil, M., 2009. Empirical bernstein bounds and sample variance
  penalization. arXiv preprint arXiv:0907.3740.

\bibitem[{Modi et~al.(2005)Modi, Shen, Tambe, and Yokoo}]{modi2005adopt}
Modi, P.~J., Shen, W.-M., Tambe, M., Yokoo, M., 2005. Adopt: Asynchronous
  distributed constraint optimization with quality guarantees. Artificial
  Intelligence 161~(1-2), 149--180.

\bibitem[{Moody et~al.(1995)Moody, Hanson, Krogh, and Hertz}]{moody1995simple}
Moody, J., Hanson, S., Krogh, A., Hertz, J.~A., 1995. A simple weight decay can
  improve generalization. Advances in neural information processing systems 4,
  950--957.

\bibitem[{Navia-V{\'a}zquez et~al.(2006)Navia-V{\'a}zquez, Gutierrez-Gonzalez,
  Parrado-Hern{\'a}ndez, and Navarro-Abellan}]{navia2006distributed}
Navia-V{\'a}zquez, A., Gutierrez-Gonzalez, D., Parrado-Hern{\'a}ndez, E.,
  Navarro-Abellan, J., 2006. Distributed support vector machines. IEEE
  Transactions on Neural Networks 17~(4), 1091--1097.

\bibitem[{Nocedal and Wright(2006)}]{nocedal2006numerical}
Nocedal, J., Wright, S., 2006. Numerical optimization. Springer Science \&
  Business Media.

\bibitem[{Ochs et~al.(2015)Ochs, Dosovitskiy, Brox, and
  Pock}]{ochs2015iteratively}
Ochs, P., Dosovitskiy, A., Brox, T., Pock, T., 2015. On iteratively reweighted
  algorithms for nonsmooth nonconvex optimization in computer vision. SIAM
  Journal on Imaging Sciences 8~(1), 331--372.

\bibitem[{Perez-Cruz and Kulkarni(2010)}]{perez2010robust}
Perez-Cruz, F., Kulkarni, S.~R., 2010. Robust and low complexity distributed
  kernel least squares learning in sensor networks. IEEE Signal Processing
  Letters 17~(4), 355--358.

\bibitem[{Pottie and Kaiser(2000)}]{pottie2000wireless}
Pottie, G.~J., Kaiser, W.~J., 2000. Wireless integrated network sensors.
  Communications of the ACM 43~(5), 51--58.

\bibitem[{Predd et~al.(2006)Predd, Kulkarni, and Poor}]{predd2006distributed}
Predd, J., Kulkarni, S., Poor, H., 2006. Distributed learning in wireless
  sensor networks. IEEE Signal Processing Magazine 23~(4), 56--69.

\bibitem[{Predd et~al.(2009)Predd, Kulkarni, and Poor}]{predd2009collaborative}
Predd, J.~B., Kulkarni, S.~R., Poor, H.~V., 2009. A collaborative training
  algorithm for distributed learning. IEEE Transactions on Information Theory
  55~(4), 1856--1871.

\bibitem[{Quinlan(1993)}]{quinlan1993combining}
Quinlan, J.~R., 1993. Combining instance-based and model-based learning. In:
  Proceedings of the Tenth International Conference on Machine Learning. pp.
  236--243.

\bibitem[{Rogers et~al.(2011)Rogers, Farinelli, Stranders, and
  Jennings}]{rogers2011bounded}
Rogers, A., Farinelli, A., Stranders, R., Jennings, N.~R., 2011. Bounded
  approximate decentralised coordination via the max-sum algorithm. Artificial
  Intelligence 175~(2), 730--759.

\bibitem[{Sak et~al.(2014)Sak, Vinyals, Heigold, Senior, McDermott, Monga, and
  Mao}]{sak2014squence}
Sak, H., Vinyals, O., Heigold, G., Senior, A., McDermott, E., Monga, R., Mao,
  M., 2014. Sequence discriminative distributed training of long short-term
  memory recurrent neural networks. In: Interspeech 2014.

\bibitem[{Samet and Miri(2012)}]{samet2012privacy}
Samet, S., Miri, A., 2012. Privacy-preserving back-propagation and extreme
  learning machine algorithms. Data \& Knowledge Engineering 79, 40--61.

\bibitem[{Sayed(2014)}]{sayed2014adaptive}
Sayed, A.~H., 2014. Adaptive networks. Proceedings of the IEEE 102~(4),
  460--497.

\bibitem[{Sayed et~al.(2014)}]{sayed2014adaptation}
Sayed, A.~H., et~al., 2014. Adaptation, learning, and optimization over
  networks. Foundations and Trends{\textregistered} in Machine Learning
  7~(4-5), 311--801.

\bibitem[{Scardapane et~al.(2017)Scardapane, Comminiello, Hussain, and
  Uncini}]{scardapane2017group}
Scardapane, S., Comminiello, D., Hussain, A., Uncini, A., 2017. Group sparse
  regularization for deep neural networks. Neurocomputing 241, 81--89.

\bibitem[{Scardapane et~al.(2016{\natexlab{a}})Scardapane, Fierimonte,
  Di~Lorenzo, Panella, and Uncini}]{scardapane2016distributed}
Scardapane, S., Fierimonte, R., Di~Lorenzo, P., Panella, M., Uncini, A.,
  2016{\natexlab{a}}. Distributed semi-supervised support vector machines.
  Neural Networks 80, 43--52.

\bibitem[{Scardapane et~al.(2016{\natexlab{b}})Scardapane, Wang, and
  Panella}]{scardapane2016decentralized}
Scardapane, S., Wang, D., Panella, M., 2016{\natexlab{b}}. A decentralized
  training algorithm for echo state networks in distributed big data
  applications. Neural Networks 78, 65--74.

\bibitem[{Scardapane et~al.(2015)Scardapane, Wang, Panella, and
  Uncini}]{scardapane2015distributed}
Scardapane, S., Wang, D., Panella, M., Uncini, A., 2015. Distributed learning
  for random vector functional-link networks. Information Sciences 301,
  271--284.

\bibitem[{Schmidhuber(2015)}]{schmidhuber2015deep}
Schmidhuber, J., 2015. Deep learning in neural networks: An overview. Neural
  Networks 61, 85--117.

\bibitem[{Schmidt(2010)}]{schmidt2010graphical}
Schmidt, M., 2010. Graphical model structure learning with l1-regularization.
  Ph.D. thesis, The University of British Columbia (Vancouver).

\bibitem[{Sun et~al.(2016)Sun, Scutari, and Palomar}]{sun2016distributed}
Sun, Y., Scutari, G., Palomar, D., 2016. Distributed nonconvex multiagent
  optimization over time-varying networks. In: Proceedings of the 50th annual
  Asilomar Conference on Signals, Systems, and Computers.

\bibitem[{Tibshirani(1996)}]{tibshirani1996regression}
Tibshirani, R., 1996. Regression shrinkage and selection via the lasso. Journal
  of the Royal Statistical Society. Series B (Methodological), 267--288.

\bibitem[{Vieira-Marques et~al.(2006)Vieira-Marques, Robles, Cucurull, Navarro,
  et~al.}]{vieira2006secure}
Vieira-Marques, P.~M., Robles, S., Cucurull, J., Navarro, G., et~al., 2006.
  Secure integration of distributed medical data using mobile agents. IEEE
  Intelligent Systems~(6), 47--54.

\bibitem[{Xiao and Boyd(2004)}]{xiao2004fast}
Xiao, L., Boyd, S., 2004. Fast linear iterations for distributed averaging.
  Systems \& Control Letters 53~(1), 65--78.

\bibitem[{Xiao et~al.(2007)Xiao, Boyd, and Kim}]{xiao2007distributed}
Xiao, L., Boyd, S., Kim, S.-J., 2007. Distributed average consensus with
  least-mean-square deviation. Journal of Parallel and Distributed Computing
  67~(1), 33--46.

\bibitem[{Zeiler(2012)}]{zeiler2012adadelta}
Zeiler, M.~D., 2012. Adadelta: an adaptive learning rate method. arXiv preprint
  arXiv:1212.5701.

\bibitem[{Zhang and Zhong(2013)}]{zhang2013privacy}
Zhang, Y., Zhong, S., 2013. A privacy-preserving algorithm for distributed
  training of neural network ensembles. Neural Computing and Applications
  22~(1), 269--282.

\bibitem[{Zhu and Mart\'{\i}nez(2010)}]{zhu2010discrete}
Zhu, M., Mart\'{\i}nez, S., 2010. {D}iscrete-time {D}ynamic {A}verage
  {C}onsensus. {A}utomatica 46~(2), 322--329.

\end{thebibliography}

\end{document}